\documentclass{article}

\usepackage[preprint,nonatbib]{neurips_2020}

\usepackage[utf8]{inputenc} 
\usepackage[T1]{fontenc}    
\usepackage{url}            
\usepackage{booktabs}       
\usepackage{amsfonts}       
\usepackage{nicefrac}       
\usepackage{microtype}

\usepackage[usenames,dvipsnames]{xcolor}

\usepackage{graphicx}
\usepackage{algorithm}
\usepackage{algpseudocode}

\usepackage{xr-hyper}
\usepackage{hyperref}
\usepackage{xr}

\usepackage{amsmath, amssymb, amsthm}
\usepackage{booktabs}
\usepackage{dsfont}
\usepackage{caption}
\usepackage{subcaption}
\usepackage{enumitem}
\usepackage{wrapfig}

\newtheorem{theorem}{Theorem}
\newtheorem{corollary}{Corollary}
\newtheorem{lemma}{Lemma}
\newtheorem{assumption}{Assumption}
\newtheorem{definition}{Definition}

\newcommand{\EE}{\mathbb{E}}

\newcommand{\VV}{\mathbb{V}}

\newcommand{\holder}{\,\cdot\,}

\DeclareMathOperator*{\argmin}{arg\,min}
\DeclareMathOperator*{\argmax}{arg\,max}

\newcommand{\Fcal}{\mathcal{F}}

\newcommand{\Pcal}{\Pi}
\newcommand{\Scal}{\mathcal{S}}

\newcommand{\Acal}{\mathcal{A}}
\newcommand{\Tcal}{\mathcal{T}}

\newcommand{\Vmax}{V_{\max}}
\newcommand{\RR}{\mathbb{R}}

\newcommand{\fo}{f_{k-1}}

\newcommand{\muu}{{\mu}}

\newcommand{\Lcal}{\mathcal{L}}

\newcommand{\piopt}{\pi^\star}
\newcommand{\picomp}{\widetilde{\pi}}

\newcommand{\OurT}{\widetilde{\mathcal{T}}}
\newcommand{\OurTpi}{\widetilde{\mathcal{T}}^{\pi}}
\newcommand{\OurTpit}{\widetilde{\mathcal{T}}^{\pi_t}}
\newcommand{\Sext}{\mathcal{S}'}
\newcommand{\Aext}{\mathcal{A}'}
\newcommand{\muhat}{\widehat{\mu}}
\newcommand{\avicompleteerror}{\epsilon'_{\Fcal}}
\newcommand{\apicompleteerror}{\epsilon_{\Fcal}}
\newcommand{\apigreedyerror}{\epsilon_{\Pcal}}
\newcommand{\projection}{\Xi}
\newcommand{\weakset}{\Pi_{C}^{all}}

\newcommand{\weaksettext}{$\zeta$-constrained policy set~}
\newcommand{\densitybound}{U}
\DeclareMathOperator{\E}{\mathbb{E}}
\newcommand{\what}{\ensuremath{\widehat}}

\newcommand{\badstateprob}{\epsilon_\zeta}
\newcommand{\tvmuerror}{\epsilon_{\mu}}

\newcommand{\strongset}{\Pi_{SC}^{all}}
\newcommand{\strongsettext}{strong $\zeta$-constrained policy set~}

\title{
Provably Good Batch 
Reinforcement Learning Without Great Exploration}

\author{
  Yao Liu \\
Stanford University\\
\texttt{yaoliu@stanford.edu} \\
\And
Adith Swaminathan \\
Microsoft Research \\
\texttt{adswamin@microsoft.com} \\
\AND
Alekh Agarwal \\
Microsoft Research \\
\texttt{alekha@microsoft.com} \\
\And
Emma Brunskill \\
Stanford University\\
\texttt{ebrun@cs.stanford.edu}
}

\begin{document}

\maketitle
\setlength{\abovedisplayskip}{3pt}
\setlength{\belowdisplayskip}{3pt}

\begin{abstract}
  Batch reinforcement learning (RL) is important to apply RL algorithms to many high stakes tasks. Doing batch RL in a way that yields a reliable new policy in large domains is challenging: a new decision policy may visit states and actions outside the support of the batch data, and function approximation and optimization with limited samples can further increase the potential of learning policies with overly optimistic estimates of their future performance. Recent algorithms have shown promise but can still be overly optimistic in their expected outcomes. Theoretical work that provides strong guarantees on the performance of the output policy relies on a strong concentrability assumption, that makes it unsuitable for cases where the ratio between state-action distributions of behavior policy and some candidate policies is large. This is because in the traditional analysis, the error bound scales up with this ratio. We show that a small modification to Bellman optimality and evaluation back-up to take a more conservative update can have much stronger guarantees. In certain settings, they can find the approximately best policy within the state-action space explored by the batch data, without requiring a priori assumptions of concentrability. We highlight the necessity of our conservative update and the limitations of previous algorithms and analyses by illustrative MDP examples, and demonstrate an empirical comparison of our algorithm and other state-of-the-art batch RL baselines in standard benchmarks.

\end{abstract}

\section{Introduction}
A key question in reinforcement learning is about learning good policies from off policy batch data in large or infinite state spaces. This problem is not only relevant to the batch setting; many online RL algorithms use a growing batch of data such as a replay buffer \cite{lin1992self, mnih2015human}. Thus understanding and advancing batch RL can help unlock the potential of large datasets and may improve online RL algorithms. In this paper, we focus on the algorithm families based on Approximate Policy Iteration (API) and Approximate Value Iteration (AVI), which form the prototype of many model-free online and offline RL algorithms. 
In large state spaces, function approximation is also critical to handle state generalization. However, the deadly triad \cite{sutton2018reinforcement} of off-policy learning, function approximation and bootstrapping poses a challenge for model-free batch RL. 
One particular issue is that the max in the Bellman operator may pick $(s,a)$ pairs in areas with limited but rewarding samples, which can lead to overly optimistic value function estimates and under-performing policies.  

This issue has been studied from an algorithmic and empirical perspective in many ways. Different heuristic approaches \cite{fujimoto2019off,laroche2019safe,kumar2019stabilizing} have been proposed and shown to be effective to relieve this weakness 
empirically. However the theoretical analysis of these methods are limited to tabular problem settings, and the practical algorithms also differ substantially from their theoretical prototypes. 

Other literature focuses primarily on approaches that have strong theoretical guarantees. Some work has focused on safe batch policy 
improvement: only deploying new policies if with high confidence they improve over prior policies. However, such work either 
assumes that the policy class can be enumerated~\cite{thomas2019preventing}, which is infeasible in a number of important cases, or has used regularization with a behavior policy~\cite{thomas2015higha} as a heuristic, 
which can disallow identifying significantly different but better 
policies. On the other hand over the last two decades there have been a number of formal analyses of API and AVI algorithms in batch settings with large or infinite state and policy spaces \cite{munos2003error,munos2005error,antos2008learning,munos2008finite,farahmand2016regularized,chen2019information, xie2020q}. Strong assumptions about the distribution of the batch data is usually assumed, often referred to as \textit{concentrability} assumptions. They ensure that the ratio between the induced state-action distribution of \textit{any non-stationary} policy and the state-action distribution of the batch data is upper bounded by a constant, referred to as the \textit{concentrability} coefficient. This is a strong assumption and hard to verify in practice since the space of policies and their induced state-action distributions is huge. For example, in healthcare datasets about past physician choices and patient outcomes, decisions with a poor prognosis may be very rare or absent for a number of patient conditions. 
This could result in the scale of concentration coefficient, and therefore existing performance bounds~\cite{munos2003error,munos2005error,chen2019information,antos2008learning}, being prohibitively large or unbounded. 
This issue occurs 
even if good policies (such as another physician's decision-making 
policy) 
are well supported in the dataset.

In the on-policy setting, various relaxations of the concentrability assumption have been studied. For example, some methods~\cite{kakade2002approximately,Kakade01,agarwal2019optimality} obtain guarantees scaling in the largest density ratio between the optimal policy and an initial state distribution, which is a  significantly milder assumption than  concentrability~\cite{scherrer2014approximate,agarwal2019optimality}. Unfortunately, 
leveraging a similar assumption is not straightforward in the fully offline batch RL setting, where an erroneous estimate of a policy's quality in a part of the state space not supported by data would never be identified through subsequent online data collection. 

\textbf{Our contributions.} Given these considerations,  an appealing goal is to ensure that we will output a good policy with respect to the batch data available. 
Rather than assuming concentrability on the entire policy space, we algorithmically choose to focus on policies that satisfy a bounded density ratio assumption akin to the on-policy policy gradient methods, and  successfully compete with such policies. Our methods are guaranteed to converge to the approximately best decision policy in this set, and our error bound scales with a parameter that defines this policy set. 
If the behavior data provides sufficient coverage of states and actions visited under an optimal policy, then our  algorithms will output a near-optimal policy. In the physician example and many real-world scenarios,  good policies are often well-supported by the behavior distribution even when the concentrability assumption fails.

Many recent state-of-the-art batch RL algorithms~\cite{fujimoto2019off,laroche2019safe,kumar2019stabilizing} do not provide such guarantees and can struggle, as we show shortly with an illustrative example.   
Our methods come from a minor modification of adding conservatism in the policy evaluation and improvement operators for API. The key insight is that prior works are enforcing an insufficient constraint; we instead constrain the algorithm to select among actions which have good coverage according to the marginalized support of \emph{states and actions} under the behavior dataset.  Our API algorithm ensures the desired ``doing the best with what we've got'' guarantee, and we provide a slightly weaker result for an AVI-style method. 

We then validate a practical implementation of our algorithm in a discrete and a continuous control task. It achieves better and more stable performance compared with baseline algorithms. This work makes a concrete step forward on providing guarantees on the quality of batch RL with function approximation that apply to realistic settings.

\section{Problem Setting}
Let $M = <\Scal, \Acal, P, R, \gamma, \rho>$ be a Markov Decision Process (MDP), where $\Scal$, $\Acal$, $P$, $R$, $\gamma$, $\rho$ are the state space, action space, dynamics model, reward model,  discount factor and distribution over initial states, respectively. 
A policy $\pi~:~\Scal\to \Delta(\Acal)$ is a conditional distribution over actions given state. We often use $\pi$ to refer to its probability mass/density function according to whether $\Acal$ is discrete or continuous. 
A policy $\pi$ and MDP $M$ will introduce a joint distribution of trajectories: $s_0, a_0, r_0, s_1, a_1, \dots$, where $s_0 \sim \rho(\cdot)$, $a_h \sim \pi(s_h), r_h \sim R(s_h,a_h)$, $s_{h+1} \sim P(s_h,a_h)$. The expected discounted sum of rewards of a policy $\pi$  given an initial state-action pair $s$ is  \mbox{$V^{\pi}(s) = \mathbb{E} \left[ \sum_{h=0}^{\infty} \gamma^h r_h \right]$}, and $Q^{\pi}(s,a)$ further conditions on the first action being $a$. 
We assume for all possible $\pi$, $Q^{\pi}(s,a) \in [0, \Vmax]$. 
We define $v^{\pi}$ to be the expectation of $V^\pi(s)$ under initial state distribution. We are given a dataset $D$ with $n$ samples \emph{drawn i.i.d. from a behavior distribution} $\mu$ over $\Scal \times \Acal$, and overload notation to denote the marginal distribution over states by $\mu(s) = \sum_{a \in \Acal} \mu(s,a)$. 
Approximate value and policy iteration (AVI/API) style algorithms fit a $Q$-function over state, action space: $f: \Scal \times \Acal \to [0, \Vmax]$. Define the \emph{Bellman optimality/evaluation operators} $\Tcal$ and $\Tcal^\pi$ as:
\begin{align*} 
	\hspace*{-0.2cm}(\Tcal f)(s,a) := r(s,a) + \gamma \E_{s'} \left[\max_{a'}f(s',a')\right] \,\mbox{and}\, (\Tcal^\pi f) (s,a) := r(s,a) + \gamma \E_{s'}\mathbb{E}_{a' \sim \pi} f(s',a').
\end{align*}
Define $\what \Tcal$ and $\what \Tcal^\pi$ 
by replacing expectations with sample averages. Then 
AVI iterates $f_{k+1} \leftarrow \what\Tcal f_k$. 
API performs policy evaluation by iterating $f_{i+1} \leftarrow \what\Tcal^{\pi_k} f_i$ until convergence to get $f_{k+1}$ followed by a greedy update to a deterministic policy: \mbox{$\pi_{k+1}(s) = \argmax_a f_{k+1}(s,a)$}. 

\section{Challenges for Existing Algorithms}
\label{sec:challenges}
\begin{figure}[t!]
    \begin{minipage}{0.35\textwidth}
    \centering
    \includegraphics[width=\textwidth]{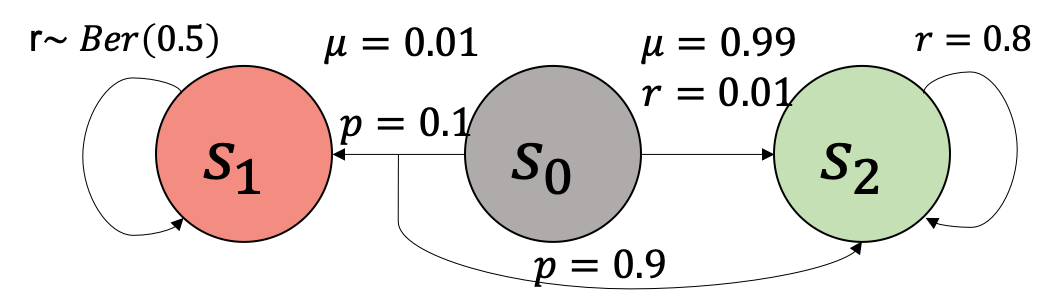}
    \subcaption{MDP with a rare transition}
    \label{fig:example1}
    \end{minipage}
    \hspace{\fill}
    \begin{minipage}{0.55\textwidth}
    \centering
    \includegraphics[width=\textwidth]{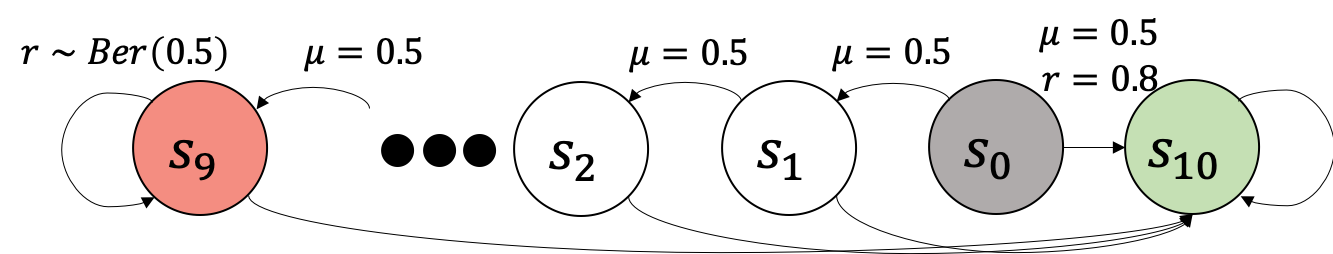}
    \subcaption{Combination lock}
    \label{fig:example2}
    \end{minipage}
    \caption{Challenging MDPs for prior approaches. In both MDPs, the episode starts from $s_0$ and ends after a fixed horizon, 2 or 10 respectively. The optimal policy is to reach the green state. Transition probabilities are labeled as $p$ on edges or are deterministic. $\mu$ is the conditional probability of action given state when generating the data.  
    The reward distribution is labeled on the edges unless it is $0$.}
    \label{fig:examples}
    \vspace{-0.2in}
\end{figure}

Value-function based batch RL typically uses AVI- or API-style algorithms described above.
A standard assumption used in theoretical analysis for API and AVI is the 
\emph{concentrability} assumption~\cite{munos2003error,munos2005error,chen2019information} that posits: for \textit{any} distribution $\nu$ that is reachable for some non-stationary policy, $\|\nu(s,a)/\mu(s,a)\|_{\infty} \le C$. Note that even if $C$ might be reasonably small for many high value and optimal policies in a domain, there may exist some policies which force $C$ to be exponentially large (in the number of actions and/or effective horizon). Consider a two-arm bandit example where the good arm has a large probability under behavior policy. In this example it should be easy to find a good policy supported by the behavior data but there exist some policies which are neither good nor supported, but the theoretical results requires setting an upper bound $C$ for admitting policies choosing the bad arm.  (See detailed discussions on the unreasonableness of concentrability~\cite{chen2019information,scherrer2014approximate,agarwal2019optimality}.)

\begin{figure}[th]
\vspace{-\intextsep}
\hspace*{-.75\columnsep}
\centering
    \begin{minipage}{0.4\textwidth}
    \centering
    \includegraphics[width=\textwidth]{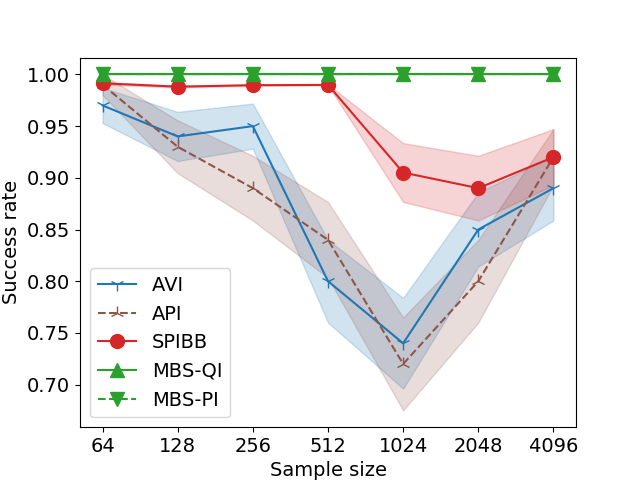}
    \subcaption{MDP with a rare transition}
    \label{fig:result1}
    \end{minipage}
    \hspace{1cm}
    \begin{minipage}{0.4\textwidth}
    \centering
    \includegraphics[width=\textwidth]{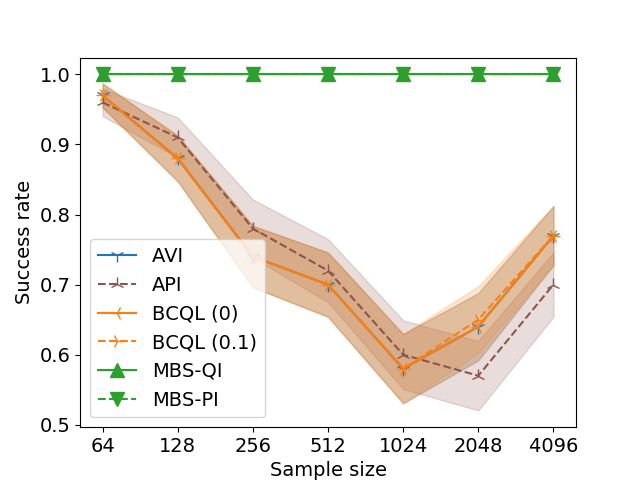}
    \subcaption{2-arm combination lock}
    \label{fig:result2}
    \end{minipage}
    \caption{Frequency of converging to $\pi^\star$ (Success Rate) for different variants of approximated value iteration: AVI, API, BCQL \cite{fujimoto2019off} (with different thresholds), SPIBB~\cite{laroche2019safe}, our algorithms (MBS-PI, MBS-QI). Tabular algorithm of BEAR \cite{kumar2019stabilizing} is same as BCQL with non-zero threshold. Frequencies are computed from 100 runs, and the change of frequencies over sample size is shown across $X$-axis. Error bars are standard deviation. We use a MLE estimate as $\muhat$ and $b=10/\text{sample size}$.}
    \vspace{-0.1in} 
    \label{fig:results}
\end{figure}

Nevertheless, the standard concentrability  assumption has been employed in many prior works on API \cite{lazaric2012finite,munos2003error} and AVI  \cite{munos2005error,munos2008finite,farahmand2010error, chen2019information}. In practice, these algorithms often diverge~\cite{riedmiller2005neural,hafner2011reinforcement}, likely due to the failure of this assumption. Such empirical observations have helped motivate several recent AVI (e.g.~\cite{fujimoto2019off,kumar2019stabilizing}) and API (e.g.~\cite{laroche2019safe}) works that improve stability and performance.  The BCQL algorithm~\cite{fujimoto2019off} only bootstraps value estimates from actions with conditional probabilities above a threshold under the behavior policy.\footnote{In their tabular algorithm BCQL the threshold is zero but we extend it to non-zero which is also more consistent with their deep RL variant BCQ.} The BEAR algorithm \cite{kumar2019stabilizing} uses distribution-constrained backups as its prototype for theoretical analysis which, in the tabular setting, is essentially BCQL using non-zero threshold. SPIBB~\cite{laroche2019safe} follows the behavior policy in less explored state-action pairs while attempting improvement everywhere else, assuming they know the behavior policy. However, we find that these algorithms have failure modes even in simple MDPs shown in Figure~\ref{fig:examples}. Specifically, BCQL and BEAR perform bootstrapping based on \emph{just the action probability}, even if the state in question itself is less explored.\footnote{This corresponds to a small $f(\epsilon)$ in Theorem 4.2 of Kumar et al.~\cite{kumar2019stabilizing}, hence in line with their theory.} This causes them to bootstrap values from infrequent states in the data, leading to bad performance as shown in Figure~\ref{fig:result2}. SPIBB is robust to this failure mode when the behavior policy is known. In this paper, we do not make this assumption as it is unknown in many settings (e.g., physicians' decision making policy in a medical dataset). Consequently, following the estimated behavior policy from rare transitions is dangerous as our estimates can also be unreliable there. This can yield poor performance as shown in Figure~\ref{fig:result1}. In that plot, rare events start to appear as sample size increases and baseline algorithms begin to bootstrap unsafe values. Hence we see poor success rate as we get more behavior data. As samples size get very large all algorithms eventually see enough data to correct this effect. Next we describe the design and underlying theory for our new algorithms that successfully perform in these challenging MDPs (see Figure~\ref{fig:results}).

\section{Marginalized Behavior Supported Algorithms}
Our aim is to create algorithms that are guaranteed to find an approximately optimal policy over all policies that 
only visit states and actions with sufficient visitation under $\mu$. In order to present the algorithms, we introduce some useful notation and present the classical fitted $Q$ iteration (FQI)~\cite{szepesvari2005finite} and fitted policy iteration (FPI)~\cite{antos2008learning} algorithms which are the function approximation counterparts of AVI and API respectively. For the FPI algorithm, let $\Pcal \subset (\Scal \to \Delta(\Acal))$ be a policy class. Let $\Fcal \subseteq (\Scal\times\Acal \to [0,\Vmax])$ be a $Q$ function class in both FQI and FPI. We assume $\Pcal$ and $\Fcal$ are both finite but can be very large (error bounds will depend on $\log(| \Fcal||\Pcal|)$). For any function $g$ we define shorthand $\|g\|_{p,\mu}$ to denote 
 $(\E_{(s,a)\sim \mu} g(s,a)^p)^{1/p}$. 
 FQI updates $f_{k+1} = \argmin_{f\in\Fcal} \|f - \what \Tcal f_k\|_{2,\mu}^2$ and returns the greedy policy with respect to the final $f_k$ upon termination. FPI instead iterates $f_{i+1} = \argmin_{f\in\Fcal} \|f - \what \Tcal^{\pi_k} f_i \|_{2,\mu}^2$  until convergence to get $f_{k+1}$
followed by the greedy policy update to get $\pi_{k+1}$. When a restricted policy set $\Pi$ is specified so that the greedy policy improvement is not possible, weighted classification is used to update the policy in some prior works~\cite{farahmand2015classification}.
Since both AVI and API suffer from bootstrapping errors in less visited regions of the MDP even in the tabular setting,  FQI and FPI approaches have the same drawback. 

We now show how to design more robust algorithms by constraining the Bellman update to be only over state action pairs that are sufficiently covered by $\mu$.
Implementing such a constraint requires access to the density function $\mu$, over an extremely large or infinite space. Given we have samples of this distribution, several density estimation techniques~\cite{loader1996local,malec2014nonparametric} can be used to estimate $\mu$ in practice, and here we assume we have a density function $\muhat$ which is an approximate estimate of $\mu$. In the analysis section we will specify how our error bounds scale with the accuracy of $\muhat$.  Given $\muhat$ and a threshold $b$ (as algorithm input) we define a filter function over the state-action space $\Scal \times \Acal$: 
\begin{equation}
    \zeta(s,a; \muhat, b) = \mathds{1}\left(\muhat(s,a) \ge b \right).
    \label{eq:zeta}
\end{equation}
For simplicity we write $\zeta(s,a; \muhat, b)$ as $\zeta(s,a)$ and define $\zeta \circ f(s,a) := \zeta(s,a)f(s,a)$.
\begin{table}[!t]
\vspace{-0.5cm}
 \centering
 \begin{tabular}{cc}
\begin{minipage}{0.5\textwidth}
\begin{algorithm}[H]
  \caption{MBS Policy Iteration (MBS-PI)}
  \label{alg:api}
    \begin{algorithmic}[1]
  \State {\bfseries Input:} $D$, $\Fcal$, $\Pcal$, $\muhat$, $b$ 
  \State {\bfseries Output:} $\widehat{\pi}_T$
  \For{$t=0$ {\bfseries to} $T-1$}
  \For{$k=0$ {\bfseries to} $K$} 
  \State $f_{t,k+1}\! \leftarrow \! \argmin_{f \in \mathcal{F}} \Lcal_D(f, f_{t,k}; \widehat{\pi}_{t})$
  \EndFor
  \State $\widehat{\pi}_{t+1} \! \leftarrow \! \argmax_{\pi \in \Pcal} \mathbb{E}_{D} [\mathbb{E}_{\pi}\left[ \zeta\circ f_{t,K+1} 
  \right] ]$\label{alg:policy_improvement}
  \EndFor
\end{algorithmic}
\end{algorithm}
\end{minipage}
&
\begin{minipage}{0.45\textwidth}
\begin{algorithm}[H]
  \caption{MBS $Q$ Iteration (MBS-QI)}
  \label{alg:fqi}
    \begin{algorithmic}[1]
  \State {\bfseries Input:} $D$, $\Fcal$, $\muhat$, $b$ \State {\bfseries Output:} $\widehat{\pi}_T$
  
  \For{$t=0$ {\bfseries to} $T-1$}
  \State $f_{t+1} \! \leftarrow \! \argmin_{f \in \Fcal} \Lcal_D (f; f_t)$
  \State $\widehat{\pi}_{t+1}(s) \! \! \leftarrow \! \! \argmax_{a \in \Acal} \zeta\circ f_{t+1}(s,a)$\label{alg:greedy_pi}
  \EndFor
  \vspace*{0.7cm}
\end{algorithmic}
\end{algorithm}
\end{minipage}
\end{tabular}
\vspace{-0.5cm}
\end{table}
We now introduce our Marginalized Behavior Supported Policy Iteration (MBS-PI) algorithm (Algorithm~\ref{alg:api}), a minor modification to vanilla FPI that constrains the Bellman backups and policy improvement step to have sufficient support by the provided data. The key is to change the policy evaluation operator and only evaluate next step policy value over supported $(s,a)$ pairs, and constrain policy improvement to only optimize over supported $(s,a)$ pairs defined by $\zeta$. We define the \emph{$\zeta$-constrained Bellman evaluation operator} $\OurTpi$ as, for any $f: \Scal \times \Acal \to \RR$,
\begin{align}
(\OurTpi f) (s,a) = r(s,a) + \gamma \mathbb{E}_{s'}  \textstyle\sum_{a' \in \Acal} \left[\pi(a'|s') \zeta \circ f(s',a')\right].
\end{align}
This reduces updates that may be over-optimistic  
estimates in $(s',a')$'s that are not adequately covered by $\mu$ by using the most pessimistic estimate of $0$ from such pairs. 
There is an important difference with SPIBB~\cite{laroche2019safe}, which still backs up $f$ values (but by estimated behavior probabilities) from such rarely visited state-action pairs: given limited data, this can lead to erroneous decisions (c.f. Fig.~\ref{fig:results}). 

Given a batch dataset, 
we follow the common batch RL choice of least-squares residual minimization~\cite{munos2003error} and define empirical loss of $f$ given $f'$ (from last iteration) and policy $\pi$:
\begin{align*}
    \textstyle
    \Lcal_{D}(f;f',\pi) := \EE_{D} \left( f(s,a) - r - \gamma \sum_{a' \in \Acal} \pi(a'|s')\zeta \circ f'(s',a') \right)^2.
\end{align*}
In the policy improvement step (line~\ref{alg:policy_improvement} of Algorithm~\ref{alg:api}), to ensure that the resulting policy has sufficient support, our algorithm applies the filter to the computed state-action values before performing maximization,
like classification-based API~\cite{farahmand2015classification}.
We maximize using the dataset $D$ ($\E_D$ is a sample average) and within the policy class $\Pcal$, which may not include all deterministic policies.

Analogous to MBS-PI, we introduce Marginalized Behavior Supported $Q$ Iteration (MBS-QI) (Algorithm~\ref{alg:fqi}) which constrains the Bellman backups to have sufficient support in the provided data. 
Define $\OurT$ to be a $\zeta$-constrained Bellman optimality operator: for any $f: \Scal \times \Acal \to \RR$,
\begin{align} 
    \textstyle
	(\OurT f)(s,a) := r(s,a) + \gamma \mathbb{E}_{s'} \left[\max_{a'}\zeta \circ f(s',a')\right].
\end{align}
Similarly, we define the empirical loss of $f$ given another function $f'$ as:
\begin{align*}
    \textstyle
     \Lcal_{D}(f; f') := \EE_{D} \left( f(s,a) - r - \gamma \textstyle\max_{a' \in \Acal} \zeta\circ f'(s',a') \right)^2.
\end{align*}
We also alter the final policy output step to only be able to select among actions which lie in the support set (line~\ref{alg:greedy_pi} of Algorithm~\ref{alg:fqi}). In Figure~\ref{fig:results} we showed empirically that our MBS-QI and MBS-PI can both successfully return the optimal policy when the data distribution covers the optimal policy in two illustrative examples where prior approaches struggle. The threshold $b$ is the only hyper-parameter needed for our algorithms. Intuitively, $b$ trades off the risk of extrapolation with the potential benefits from more of the batch dataset. We discuss how practitioners can set $b$ in Section~\ref{sec:expts}.

\section{Analysis}
\label{sec:theory}
We now provide performance bounds on the  policy output by MBS-PI and MBS-QI. Complete proofs for both are in the appendix. 

\subsection{Main results for MBS-PI}

We start with some definitions and assumptions. Given $\pi$, let $\eta^{\pi}_h(s)$ be the marginal distribution of $s_h$ under $\pi$, that is, $\eta^{\pi}_h(s) := \Pr[s_h = s | \pi]$, $\eta^{\pi}_h(s,a) = \eta^{\pi}_h(s)\pi(a|s)$, and $\eta^{\pi}(s,a) = (1-\gamma)\sum_{h=0}^\infty \gamma^h\eta^{\pi}_h(s,a) $. We make the following assumptions:
\begin{assumption}[Bounded densities]
\label{asm:bounded_density}
    For any non-stationary policy $\pi$ and $h \ge 0$, $\eta^{\pi}_h(s,a) \le \densitybound$.
\end{assumption}
\begin{assumption}[Density estimation error]
    With probability at least $1-\delta$, $ \|\muhat - \mu\|_\text{TV} \le \tvmuerror$.
    \label{assume:muhat}
\end{assumption}
\begin{assumption}[Completeness under $\OurTpi$]
\label{asm:api_completeness} $\forall \pi \in \Pcal$,
$\max_{f \in \Fcal} \min_{g \in \Fcal}  \| g - \OurTpi f \|_{2,\mu}^2 \le \apicompleteerror.$ 
\end{assumption}
\begin{assumption}[$\Pi$ Completeness]
\label{asm:weak_api_pi_realizalibity} $\forall f \in \Fcal$, $\min_{\pi \in \Pcal} \|  \mathbb{E}_{\pi}\left[ \zeta \circ f(s,a)\right] -  \max_{a}\zeta \circ f(s,a)\|_{1,\mu} \le \apigreedyerror$.
\end{assumption}
Assumptions \ref{asm:api_completeness} and 
\ref{asm:weak_api_pi_realizalibity} are common but adapted to our $\zeta-$filtered operators, implying that the function class chosen is approximately complete with respect to our operator and that policy class can approximately recover the greedy policy. Assumptions 
\ref{asm:bounded_density} and \ref{assume:muhat} are novel. Assumption \ref{assume:muhat} bounds the accuracy of estimating the state--action behavior density function from limited data:  $1/\sqrt{n}$ errors are standard using maximum likelihood estimation for instance~\cite{zhang2006eps}, and the size of this error appears in the bounds. Finally 
Assumption 
\ref{asm:bounded_density} that the  probability/density of any marginal state distribution is bounded is not very restrictive.\footnote{We need this assumption only because we filter based on $\hat\mu$ instead of $\eta^\pi/\mu$ during policy improvement.}
For example, this assumption holds when the density function of transitions $p(\cdot|s,a)$ and the initial state  distribution are both bounded. 
This is always true for discrete spaces, and also holds for many distributions in continuous spaces including Gaussian distributions. 

The dataset may not have sufficient samples of state--action pairs likely under the optimal policy in order to reliably and confidently return the optimal policy.  Instead we hope and will show that our methods will return a policy that is close to the best policy which has sufficient support in the provided dataset.  More formally, given a $\zeta$ filter over state--action pairs, we define a set of  policies: 
\begin{definition}[\weaksettext] 
\label{def:weak_constrained_set}
Let $\weakset$ be the set of policies $\Scal \to \Delta(\Acal)$ such that $\Pr(\zeta(s,a) = 0|\pi) \le \badstateprob$. That is, $\EE_{s,a\sim \eta^{\pi}} \left[ \mathds{1} \left( \zeta(s,a) = 0\right) \right] \le \badstateprob$.
\end{definition} 
$\badstateprob$ defines the  probability under a policy of escaping to  state-actions with insufficient data during an  episode.  $\weakset$ adapts its size w.r.t. the filter $\zeta$ which is a function of the hyper-parameter $b$, and $\weakset$ does not need to be contained in $\Pi$. We now lower bound the value of the policy returned by  MBS-PI by the value of any $\picomp \in \weakset$ up to a small error, which implies a  small error w.r.t. the policy with the highest value in $\weakset$. For ease of notation, we will denote $C=U/b$ in our results and analysis, being clear that $C$ is not assumed (as in concentrability) and is simply a function of the hyper-parameter $b$. 
\begin{theorem}[Comparison with best covered policy]
\label{thm:api}
With Assumption \ref{asm:bounded_density} and \ref{assume:muhat} holds, given an MDP $M$, a dataset $D = \{ (s,a,r,s') \}$ with $n$ samples drawn i.i.d. from $\mu \times R \times P$, and a Q-function class $\Fcal$ and a policy class $\Pcal$ satisfying Assumption \ref{asm:api_completeness} and \ref{asm:weak_api_pi_realizalibity}, $\widehat{\pi}_t$ from Algorithm \ref{alg:api}
satisfies that w. p. at least $1-3\delta$, 
\begin{align*}
\hspace*{-0.2cm}
v^{\picomp}_M - v^{\widehat{\pi}_t}_M \le
 \mathcal{O}\left(\frac{C\sqrt{\Vmax^2  \ln(|\Fcal||\Pcal|/\delta)}}{(1-\gamma)^3\sqrt{n}} \right) + \frac{8C\sqrt{\apicompleteerror}+6C\Vmax\tvmuerror }{(1-\gamma)^3} + \frac{2C \apigreedyerror + 3\gamma^{K-1} \Vmax}{(1-\gamma)^2} + \frac{\Vmax\badstateprob}{1-\gamma},
\end{align*}
for any policy $\picomp \in \weakset$ under Assumptions~\ref{asm:bounded_density} and~\ref{assume:muhat} and any $t \geq K$. $C = U/b$. $K$ is the number of policy evaluation iterations (inner loop) and $t$ is the number of policy improvement steps.
\end{theorem}

\textit{Proof sketch}. 
The key is to characterize the filtration effect of the conservative Bellman operators in terms of the resulting value function estimates. As an analysis tool, we construct an auxiliary MDP $M'$ by adding one additional action $a_{\text{abs}}$ in each state, leading to a zero reward absorbing state $s_\text{abs}$. For a \emph{subset} of policies in $M'$, the fixed point of our conservative operator $\OurTpi$ is $Q^\pi$ in $M'$. For that subset, we also have a bounded density ratio, thus we can provide the error bound in $M'$. For any $\zeta$-constrained policy, we show that it can be mapped to that subset of policies in $M'$ without much loss of value, and then we finish the proof to yield a bound in $M$. \qedsymbol 

Our results match in $n$ the fast rate error bound of API/AVI \cite{farahmand2011regularization, pires2012statistical, chen2019information,lazaric2012finite}. 
The dependency on horizon also matches the standard API analysis
and is $\mathcal{O}(1/(1-\gamma)^3)$. 
The dependency on $\densitybound/b$ is same as the dependency on concentrability coefficient for vanilla API analysis, but our guarantee adapts given a hyper-parameter choice instead of imposing a worst case bound over all policies.  Unlike Theorem 4.2 of Kumar et al.~\cite{kumar2019stabilizing}, there is no uncontrolled $f(\epsilon)$-like term. We avoid that term by placing stronger restrictions in our updates and in the comparator class $\weakset$: note that if Kumar et al. were to use our comparator class $\weakset$, they would not recover our guarantees (see also discussion of BEAR in Section~\ref{sec:challenges}). 
When there is an optimal policy with small escaping probability from support of $\zeta$, we obtain an error bound relative to the optimal value.
\begin{corollary}[$\mu$ covers an optimal policy]
\label{cor:api_optgap}
If there exists an optimal policy $\pi^\star$ in $M$ such that $\Pr(\mu(s,a) \le 2b|\piopt) \le \epsilon $.
then under the conditions in Theorem \ref{thm:api}, $\widehat{\pi}_t$ returned by  Algorithm \ref{alg:api} satisfies that with probability at least $1-3\delta$, $v^{\widehat{\pi}_t}_{M} \ge v^{\star}_{M} -\Delta$, where $\Delta$ is the right hand side of Theorem \ref{thm:api} and $\badstateprob$ in $\Delta$ is $\epsilon + C\tvmuerror$.
\end{corollary}

When the completeness assumptions holds without error and $\gamma^{K-1} \leq \epsilon$, the error bound reduces to
\begin{align*}
    v^{\widehat{\pi}_t}_{M} \ge v^{\star}_{M} - \mathcal{O}\left(C\sqrt{  \ln(|\Fcal||\Pcal|/\delta)/n} + C\tvmuerror  + \epsilon\right)\Vmax/(1-\gamma)^3
\end{align*}
 This is similarly tight as prior analysis assuming concentrability~\cite{lagoudakis2003least,bertsekas2011approximate}. For MBS-QI, the error bound is in a similar form. However the general MBS-QI bound of Corollary~\ref{cor:fqi_optgap} has an additional Bellman residual term related to $\OurT$ and $\pi^\star$. This term arises since in value iteration, the fixed point of $\OurT$ may no longer be the value function of the optimal policy under support. When comparing to on-policy policy optimization~\cite{kakade2002approximately,scherrer2014approximate,agarwal2019optimistic,geist2019theory}, the constant $C$ is akin to the density ratio with respect to an optimal policy in those works, though we pay an additional price on the size of densities. The terms regarding value function completeness can be avoided by Monte-carlo estimation in on-policy settings and there is no need for density estimation to regularize value bootstrapping either. 
 
 Note that if we can find a $b$ and $\badstateprob$ such that $\mu \in \weakset$ given sufficient data, then Theorem~\ref{thm:api} immediately yields a policy improvement guarantee too, analogous to the tabular guarantees known for BCQL, SPIBB and other safe policy improvement guarantees. Here, we provide the safe policy improvement guarantees in tabular settings. 
 The details of proof are in Appendix \ref{app:api_safe_policy_improvement}.
 \begin{corollary}[Safe policy improvement -- discrete state space]
 \label{cor:spi_tabular}
For finite state action spaces and $b \le \mu_{\min}$, under the same assumptions as Theorem \ref{thm:api}, there exist function sets $\Fcal$ and $\Pcal$ (specified in the proof) such that $\widehat{\pi}_t$ from Algorithm \ref{alg:api} satisfies that with probability at least $1-3\delta$,
\begin{align*}
    v^{\hat{\pi}_t}_M \ge v^{\mu}_M -  \widetilde{\mathcal{O}}\left( \frac{\Vmax}{b(1-\gamma)^3} \left( \frac{|\Scal||\Acal|}{n} + \sqrt{\frac{|\Scal||\Acal|}{n}} \right) + \frac{\gamma^{K} \Vmax}{(1-\gamma)^2} \right)
\end{align*} 
\end{corollary}
This corollary is comparable with the safe policy improvement result in \cite{laroche2019safe} in its $\Scal$, $\Acal$ dependence. Note that their hyper-parameter $N_{\land}$ is analogous to $bn$ in our result. Our dependency on $1-\gamma$ is worse, as our algorithm is not designed for the tabular setting, and matches prior results in the function approximation setting as remarked after Theorem~\ref{thm:api}.

\subsection{Main results for MBS-VI}

Now we show the error bound of our value iteration algorithm, MBS-QI. First we introduce a similar completeness assumption about the Bellman optimality operator:
\begin{assumption}[Completeness under $\OurT$] \label{asm:avi_completeness}
	$\max_{f\in\Fcal} \min_{g \in \Fcal} \|g - \OurT f\|_{2, \mu}^2 \le \avicompleteerror$
\end{assumption}
The main theorem for value iteration is similar with the policy iteration error bound
\begin{theorem}
\label{thm:fqi}
Under Assumption \ref{asm:bounded_density} and \ref{assume:muhat}, given a MDP $M = <\Scal, \Acal, R, P, \gamma, p>$, a dataset $D = \{ (s,a,r,s') \}$ with $n$ samples that is draw i.i.d. from $\mu \times R \times P$, and a finite Q-function classes $\Fcal$ satisfying Assumption \ref{asm:avi_completeness}, $\widehat{\pi}_t$ from Algorithm \ref{alg:fqi} satisfies that with probability at least $1-\delta$,
\begin{align*}
    v^{\picomp} - v^{\widehat{\pi}_t} \le \frac{2C\left(\mathcal{O}\left(\sqrt{\frac{\Vmax^2  \ln(|\Fcal|/\delta)}{n}} \right) + \sqrt{\avicompleteerror}+\Vmax\tvmuerror + \left\| Q^{\picomp} - \OurT Q^{\picomp} \right\|_{2,\mu}\right) }{(1-\gamma)^2}  + \frac{(2\gamma^t + \badstateprob) \Vmax}{1-\gamma} 
\end{align*}
for any policy $\picomp \in \weakset$.
\end{theorem}

The error bound has similar dependency on $C$, $\avicompleteerror$, $n$ as in the API case. The dependency on horizon is $\mathcal{O}(1/(1-\gamma)^2)$ which is standard in approximate value iteration analysis \cite{munos2005error,munos2008finite}. The term $\| \zeta (Q^{\projection(\picomp)} -  \OurT Q^{\projection(\picomp)}) \|_{2,\mu}$ in the error bound in Theorem \ref{thm:fqi} comes from the fact that the bound applies to any policy $\picomp$ in the \weaksettext, which might not be optimal. The scale of this term approximately measures the sub-optimality of the policy $\picomp$. $\zeta$ implies that we only measure the sub-optimality in state-action pairs visited often enough in the data. By choosing the best $\picomp$ in $\weakset$ with a good enough $Q^{\picomp}$ that term should be relatively small, and intuitively it measures ``how good is the best policy with sufficient coverage''. This term is not universally guaranteed to be small, but empirically it can often be negligible, as our illustrative examples in Section~\ref{sec:challenges} demonstrate. In Appendix~\ref{appendix:fqi} we develop a stronger notion of a constrained policy set such that, if $\pi^*$ is in that set then this error term is guaranteed to be small. Similar to Corollary \ref{cor:api_optgap}, the same type of bound w.r.t. optimal policy can be derived, formally shown in Corollary \ref{cor:fqi_optgap}.

\begin{corollary}
\label{cor:fqi_optgap}
If there exists an $\pi^\star$ on $M$ such that $\Pr(\mu(s,a) \le 2b|\piopt) \le \epsilon $.
then under the condition as Theorem \ref{thm:fqi}, $\widehat{\pi}_t$ from Algorithm \ref{alg:fqi} satisfies that with probability at least $1-2\delta$, $v^{\widehat{\pi}_t}_{M} \ge v^{\star}_{M} -\Delta$, where $\Delta$ is the right hand side of Theorem \ref{thm:fqi} and $\badstateprob$ in $\Delta$ is $\epsilon + C\tvmuerror$.
\end{corollary}

To summarize, our analysis makes two main contributions. First, compared to prior work which provides guarantees of the performance relative to the optimal policy under concentrability, we provide similar bounds by only requiring that an optimal policy, instead of every policy, is well supported by $\mu$. Second, when an optimal policy is not well supported, our algorithms are guaranteed to output a policy whose performance is near optimal in the supported policy class.

\section{Experimental Results}
\label{sec:expts}
The key innovation in our algorithm uses an estimated $\zeta$ to filter backups from unsound state-action pairs. In Section ~\ref{sec:challenges} we showed how prior approaches without this can fail in some illustrative examples. We now experiment in two standard domains -- Cartpole and Mujoco, that utilize very different $\zeta$-estimation procedures. We show several experiments where the data collected ranges from being inadequate for batch RL to complete coverage (where even unsafe batch RL algorithms can succeed). 
Our algorithms return policies better than any baseline batch RL algorithm across the entire spectrum of datasets. Our algorithms need the hyperparameter $b$ to trade off conservatism of a large $b$ (where the algorithm stays at its initialization in the limit) and unfounded optimism of $b=0$ (classical FQI/FPI).
In discrete spaces, we can set $b=n_0/n$ where $n_0$ is our prior for the number of samples we need for reliable distribution estimates and $n$ is the total sample size. In continuous spaces, we can set the threshold to be a percentile of $\muhat$, so as to filter out updates from rare outliers in the dataset. We can also run post-hoc diagnostics on the choice of $b$ by computing the average of $\zeta(s,\pi(s))$ for the resulting policy $\pi$ over the batch dataset. If this quantity is too small, we can conclude that $\zeta$ filters out too many Bellman backups and hence rerun the procedure with a lower $b$. 

\subsection{MBS-QI in discrete spaces}
\label{sec:discrete_cartpole}
We first compare MBS-QI with FQI, BCQL \cite{fujimoto2019off}, SPIBB \cite{laroche2019safe} and behavior cloning (BC) in OpenAI gym CartPole-v0, with a discrete state space by binning each state dimension resulting in $10000$ states. FQI uses a vanilla Bellman operator to update the $Q$ function in each iteration. For BCQL, we fit a modified operator in their Eq (10), changing the constraint $\muhat(a|s) > 0$ 
to allow a threshold $b \ge 0$. For SPIBB, we learn the $Q$ function by fitting the operator in their Eq (9) (SPIBB-DQN objective). Both SPIBB and MBS-QI require $\muhat(s,a)$ to construct the modified operator, but use different notions for being conservative. For all algorithms, $\muhat(s,a)$ and $\muhat(a|s)$ are constructed using empirical counts. 

\begin{figure*}[tb]
    \includegraphics[width=\textwidth]{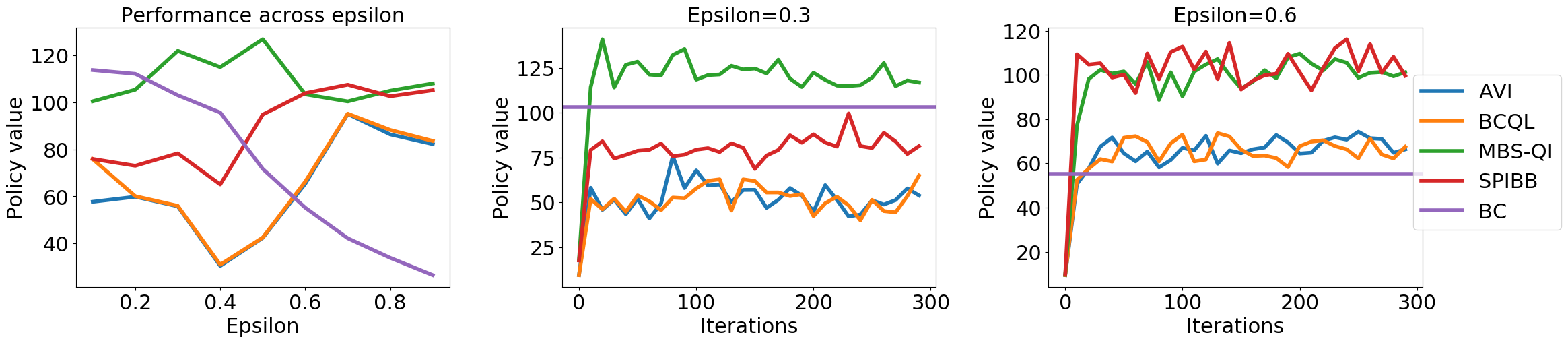}
    \caption{CartPole-v0. Left: convergent policy value across different ($\epsilon$-greedy) behavior policies. Middle and Right: learning curves when $\epsilon=0.3, 0.6$. We allow non-zero threshold for BCQL to subsume the tabular algorithm of BEAR~\cite{kumar2019stabilizing}.}
    \label{fig:cp}
    \vspace{-0.2in}
\end{figure*}

We collect $n=10^4$ transitions from an epsilon-greedy policy w.r.t a near-optimal $Q$ function. We report the final policy value from different algorithms for epsilon from $0.1$ to $0.9$ in Figure \ref{fig:cp} (left), and learning curves for $\epsilon=0.3, 0.6$ in Figure \ref{fig:cp} (middle, right). The results are averaged over 10 random seeds. Notice that BCQL, SPIBB, and our algorithm need a threshold of $\mu(s,a)$ or $\mu(a|s)$ as hyper-parameter. We show the results with the best threshold from $\{0.005, 0.001, 0.0001\}$ for $\mu(s,a)$ (MBS-QI, SPIBB) and $\{0, 0.1\}$ for $\mu(a|s)$. Our algorithm picks larger $b$ for smaller $\epsilon$ which matches the intuition of more conservatism when the batch dataset is not sufficiently diverse. The results show that our algorithm achieves good performance among all different $\epsilon$ values, and is always better than both behavior cloning and vanilla FQI unlike other baselines. 

\subsection{MBS-QL in continuous spaces}
The core argument for MBS is that $\zeta$-filtration should focus on \emph{state-action} $\muhat(s,a)$ distributions rather than $\muhat(a|s)$. To test this argument in a more complex domain, we introduce MBS Q Learning (MBS-QL) for continuous action spaces which incorporates our $\zeta$-filtration 
on top of the BCQ architecture \cite{fujimoto2019off}. MBS-QL (like BCQ) employs an actor in continuous action space and a variational auto-encoder (VAE) to approximate $\mu(a|s)$. We use an additional VAE to fit the marginal state distribution in the dataset. Since the BCQ architecture already prevents backup from $(s,a)$ with small $\mu(a|s)$, we construct an additional filter function $\zeta(s)$ on state space by $\zeta(s) = \mathds{1}(ELBO(s) > P_{2}) $ where $ELBO(s)$ is the evidence lower bound optimized by VAE, and $P_{2}$ is the $2^{nd}$ percentile of ELBO values of $s$ in the whole batch. Thus the only difference between MBS-QL and BCQ is that MBS-QL applies additional filtration $\zeta(s')$ to the update from $s'$ in Eq (13) in \cite{fujimoto2019off}. 
\begin{wrapfigure}{R}{0.45\textwidth}
    \centering
    \includegraphics[width=0.5\textwidth]{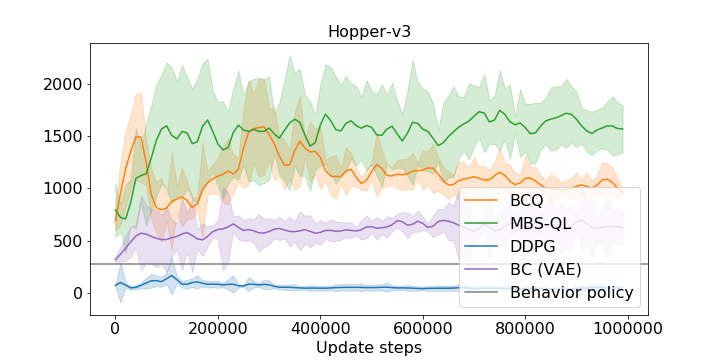}
    \caption{MuJoCo Hopper-v3 domain. Averaged over 5 random seeds and the shadow region in plot shows the standard deviation.}
    \label{fig:hopper}
\end{wrapfigure}

We compared MBS-QL, BCQ, DDPG, and behavior cloning in the Hopper-v3 domain in MuJoCo. DDPG is run with pure offline data to demonstrate the performance of a vanilla RL algorithm without adapting to the batch RL setting. Behavior cloning uses the same VAE to fit the behavior policy as MBS-QL and BCQ, and returns the learned behavior policy. The batch has $n=10^5$ transitions, and we use the same data collection policy as the ``Imperfect'' batch in BCQ~\cite{fujimoto2019off}; a near-optimal policy with two levels of noise: with probability $0.3$ it samples uniformly from action space, with probability $0.7$ it adds a Gaussian noise $\mathcal{N}(0,0.3)$ to the greedy action. In Figure \ref{fig:hopper}, BCQ is able to achieve good policy beyond both vanilla DDPG and behavior cloning method, but can not maintain this improvement as it updates its Q function in an unsafe manner. This experiment demonstrates that the additional state support filter can prevent the collapse of BCQ.

\section{Related Work}

Research in batch RL focuses on deriving the best possible policy from the available data~\cite{lange2012batch}. For practical settings that necessitate using function approximators, fitted value iteration~\cite{ernst2005tree,riedmiller2005neural} and fitted policy iteration~\cite{lagoudakis2003least} provide an empirical foundation that has spawned many successful modern deep RL algorithms. Many prior works provide error bounds as a function of the violation of realizability and completeness assumptions such as~\cite{xie2020q}. In the online RL setting, concentrability can be side-stepped~\cite{yuprovable2019} but can still pose a significant challenge (e.g., the hardness of exploration in~\cite{chen2019information}). A commonly-used equivalent form of the concentrability assumption is on the discounted summation of ratios between the product of probabilities over state actions under any policy, to the data generating distribution~\cite{antos2008learning,le2019batch}. Our goal is to relax such assumptions to make the resulting algorithms more practically useful.

Several heuristics \cite{hafner2011reinforcement,riedmiller2005neural,fujimoto2019off}
show that algorithmic modifications help alleviate the extrapolation error or maximization bias empirically.  
This paper provides a simple modification that has strong theoretical guarantees even with function approximation. In contrast, the theoretical results of BEAR \cite{kumar2019stabilizing} 
(Theorem 4.1 and 4.2) rely on a finite state space, while the algorithm analyzed in the proof is actually the same as BCQL with a  non-zero threshold whose weakness is shown in Sections \ref{sec:challenges} and \ref{sec:discrete_cartpole}. Error bound in Theorem 4.1 of \cite{kumar2019stabilizing} has an additional non-diminishing term $\alpha(\Pi)$, while non-diminishing terms in our Theorem 1 are the same as in standard analysis. Another recent work \cite{agarwal2019optimistic} highlights that sufficiently large and diverse datasets can lead to good batch RL results but we focus on the other side of the coin: a robust update rule even if we do not have a good dataset.

\section{Discussion \& Conclusion}
We study a key assumption for analysis in batch value-based RL, concentrability, and provide policy iteration and $Q$ iteration algorithms with minor modifications that can be agnostic to this assumption. We remark that the other non-standard assumption about bounded density can be relaxed if we could construct the $\zeta$ filter by thresholding density ratios directly, but this results in a different filter for each policy encountered during the algorithm's operation. Being able to threshold density ratios will also allow us to assert that $\mu \in \weakset$ always, yielding imitation and policy improvement guarantees. We anticipate future work that develops batch RL algorithms that exploit this insight. 

\section{Broader Impact}
Our improvements to batch RL may improve sample efficiency of online RL and safety of off-policy RL enough to consider them in some real-world applications. However we caution that more work is needed (e.g., in closely related areas of off-policy evaluation OPE, confidence estimation, interpretability etc.) before these methods can be reliably deployed in practice. For instance (see Figure~\ref{fig:hopper}) no batch RL algorithm has (a) low variability across runs, (b) reliably terminates when discovering a good policy, (c) gives a sharp confidence interval of expected performance. We anticipate future work in batch RL and OPE that addresses these shortcomings. 

\bibliographystyle{plain}
\bibliography{ref,refs}

\appendix
\appendix
\onecolumn

\setcounter{corollary}{0}
\setcounter{definition}{0}

In Appendix \ref{appendix:def} we introduce some basic definitions that are needed for our theoretical results. In Appendix \ref{appendix:assumption}, we provide a sufficient conditions for Assumption \ref{asm:bounded_density} that were mentioned in the main text. In Appendix \ref{appendix:api} and Appendix \ref{appendix:fqi} we prove the error bounds for MBS-PI and MBS-QI. In Appendix \ref{appendix:cartpole} and Appendix \ref{appendix:hopper} we present more details of our experimental results.

\section{Definition of auxiliary MDP and policy projection}
\label{appendix:def}
First we introduce the definition of an auxiliary MDP $M'$ based on M: each state in M has an absorbing action which leads to a self-looping absorbing state.  All the other dynamics are preserved. Rewards are 0 for the absorbing action and unchanged elsewhere. More formally: The auxiliary MDP $M'$ given $M = <\Scal, \Acal, R, P, \gamma, \rho>$ is defined as $M' = <\Sext, \Aext, R', P', \gamma, \rho>$, where $\Sext = \Scal \bigcup \{s_\text{abs}\}$, $\Aext = \Acal \bigcup \{a_\text{abs}\}$. $R'$ and $P'$ are the same as $R$ and $P$ for all $(s,a) \in \Scal \times \Acal$. $R'(s,a)$ if $s = s_\text{abs}$ or $a = a_\text{abs}$ is a point mass on $0$, and $P'(s,a)$ if $s = s_\text{abs}$ or $a = a_\text{abs}$ is a point mass on $s_\text{abs}$. A data set $D$ generated from distribution $\mu$ on $M$ is also from the distribution $\mu$ on $M'$, since all distributions on $\Scal \times \Acal$ are the same between the two MDPs. This MDP is used only to perform our analysis about the error bounds on the algorithm, and is not needed at all for executing Algorithm \ref{alg:api} and \ref{alg:fqi}. As some of the notations is actually a function of the MDP, we clarify the usage of notation w.r.t. $M$/$M'$ in the appendix:
\begin{enumerate}[nolistsep]
\item Policy value functions $V^{\pi}$/$Q^{\pi}$ and Bellman operators $\Tcal$/$\Tcal^{\pi}$ correspond to $M'$ unless they have additional subscripts.
\item The definition of $\Fcal$, $\Pcal$, $\OurT$, $\OurTpi$, $\muhat$ is independent of the change from $M$ to $M'$.
\item $\mu$ is also a distribution over $\Sext \times \Aext$. The definition of $\zeta$ will be extended to $\Sext \times \Aext$ as follow: 
\begin{align*}
    \zeta(s,a) =  \left\{
        \begin{tabular}{ll}
             $\mathds{1}\left(\muhat(s,a) \ge b\right)$ & $ s \in \Scal, a \in \Acal $ \\
             $0$ & $s = s_\text{abs}$ or $a = a_\text{abs}$
        \end{tabular} \right.
\end{align*}
(That means there is only one version of $\mu$ and $\zeta$ across $M$ and $M'$, instead of like we have $\Tcal^\pi_{M'}$ and $\Tcal^\pi_M$ for $M$ and $M'$.)
\end{enumerate}

Recall the definition of semi-norm of any function of state-action pairs. For any function $g: \Sext \times \Aext \to \mathbb{R}$, $\nu \in \Delta(\Sext \times \Aext)$, and $p\ge 1$, define the shorthand $\|g\|_{p, \nu} := (\mathbb{E}_{(s,a) \sim \nu}[|g(s,a)|^p])^{1/p}$. With some abuse of notation, later we also use this norm for $\nu \in \Delta(\Scal \times \Acal)$ (specifically, $\mu$) by viewing the probability of $\nu$ on additional $(s,a)$ pairs as zero. Given a policy $\pi$, let $\eta^{\pi}_h(s)$ be the marginal distribution of $s_h$ under $\pi$, that is, $\eta^{\pi}_h(s) := \Pr[s_h = s | s_0 \sim p, \pi]$, $\eta^{\pi}_h(s,a) = \eta^{\pi}_h(s)\pi(a|s)$, and $\eta^{\pi}(s,a) = (1-\gamma)\sum_{h=0}^\infty \gamma^h\eta^{\pi}_h(s,a) $. We also use $P(s,a)$ and $P(\nu)$ to denote the next state distribution given a state action pair or given the current state action distribution.

The norm $\|\cdot\|_{p,\nu}$ are defined over $\Sext \times \Aext$. Though for the input space of function $f \in \Fcal$ is $\Scal \times \Acal$, the norm can still be well-defined. All of the norm would not need the value of $f(s,a)$ on $s = s_\text{abs}$ or $a = a_\text{abs}$, because the distribution does not cover those $(s,a)$, or the $f$ inside of the norm is multiplied by other function that is zero for those $(s,a)$. 

We first formally state an obvious result about policy value in $M$ and $M'$. 
\begin{lemma}
\label{lem:newmdp_same_value}
For any policy $\pi$ that only have non-zero probability for $a \in \Acal$, $v^{\pi}_{M'} = v^{\pi}_{M}$.
\end{lemma}
\begin{proof}
By the definition of $M'$, $P$ and $R$ are the same with $M$ over $\Scal \times \Acal$.
$$ v^{\pi}_{M} = \mathbb{E}_{M} \left[ \sum_{t=0}^h \gamma^t r_t | s_0 \sim p, \pi\right] = \mathbb{E}_{M'} \left[ \sum_{t=0}^h \gamma^t r_t | s_0 \sim p, \pi\right] = v^{\pi}_{M'} $$
\end{proof}

For the readability we repeat the Definition \ref{def:weak_constrained_set} here

\begin{definition}[\weaksettext] 
Let $\weakset$ be the set of policies $\Scal \to \Delta(\Acal)$ such that $\Pr(\zeta(s,a) = 0|\pi) \le \badstateprob$. That is
\begin{align}
    (1-\gamma) \sum_{h=0}^\infty \gamma^h  \EE_{s,a\sim \eta^{\pi}_h} \left[ \mathds{1} \left( \zeta(s,a) = 0\right) \right] \le \badstateprob
\end{align}
\end{definition} 
Now we introduce another constrained policy set. Different from \weaksettext which we introduced in Definition \ref{def:weak_constrained_set}, this policy set is on $M'$ instead of $M$ and the policy is forced to take action $a_\text{abs}$ when $\zeta(s,a) = 0$ for all $a$. The reason we introduce this is to help us formally analyze the (lower bound of) performance of the resulting policy. We essentially treat any action taken outside of the support to be $a_\text{abs}$. Later we will define a projection to achieve that and show results about how the policy value changes after projection. 
\begin{definition}[strong $\zeta$-constrained policy set] 
Let $\strongset$ be the set of all policies $\Sext \to \Delta(\Aext)$ such that for $\forall (s,a)$ $\pi(a|s) > 0$ then 1) $\zeta(s,a) > 0$, or 2) $a = a_{\text{abs}}$.
\end{definition}
Notice that for \weaksettext we have no requirement for $\pi$ if for any action $\zeta(s,a)$ is zero. For \strongsettext we enforce $\pi$ to take action $a_\text{abs}$. The second difference is \weaksettext requires the condition holds for $s,a$ that is reachable, which means $\eta^{\pi}_h(s) > 0$ and $\pi(a|s) > 0$. Here we require the same condition holds for any $s,a$ such that $\pi(a|s) > 0$. In general, this is a stronger definition. However, we can show that for any policy in \weaksettext, it can be mapped to a policy in \strongsettext, with changing value bounds. Since we only need to change the behavior of policy in the state actions such that the state actions that $\zeta = 0$, the value of policy will not be much different.

Now we define a projection that maps any policy to $\strongset$.
\begin{definition}[$\zeta$-constrained policy projection]
\label{def:constrained_operator}
$(\projection \pi)(a | s)$ equals $\zeta(s,a)\pi(a|s)$ if $a \in \Acal$, and equals $\sum_{a' \in \Aext} \pi(a'|s) (1-\zeta(s,a'))$ if $a = a_\text{abs}$ 
\end{definition}

Next we show that the projection of policy will has an equal or smaller value than the original policy.
\begin{lemma}
\label{lem:projection_smaller_value}
For any policy $\pi: \Sext \to \Delta(\Aext)$, $v^{\pi}_{M'} \ge v^{\projection(\pi)}_{M'}$, and $v^{\pi}_{M'} = v^{\projection(\pi)}_{M'}$ if for any $(s,a)$ reachable by $\pi$, $\zeta(s,a) = 1$. 
\end{lemma}
\begin{proof}
We drop the subscription of $M'$ in this proof for ease of notation. For any given $s$,
\begin{align}
    \sum_{a \in \Aext} \pi(a|s)Q^{\projection(\pi)}(s,a) 
    &= \sum_{a \in \Acal} \pi(a|s)Q^{\projection(\pi)}(s,a) \tag{$Q^{\pi}(s,a_\text{abs}=0)$} \\
    &\ge \sum_{a \in \Acal} \zeta(a|s)\pi(a|s)Q^{\projection(\pi)}(s,a) \\
    &= \projection(\pi)(a_\text{abs}|s)Q^{\projection(\pi)}(s,a_\text{abs}) + \sum_{a \in \Acal} \projection(\pi)(a|s)Q^{\projection(\pi)}(s,a) \tag{Def of $\projection$} \\
    &= \sum_{a \in \Aext}\projection(\pi)(a|s)Q^{\projection(\pi)}(s,a) \\
    &= V^{\projection(\pi)}(s)
\end{align}
The inequality is an equality if for any $a$ s.t. $\pi(a|s) > 0$, $\zeta(s,a) = 1$.
By the performance difference lemma \cite[Lemma 6.1]{kakade2002approximately}:
\begin{align}
    v^{\projection(\pi)} - v^{\pi} = & \sum_{h=0}^\infty \gamma^h \EE_{s\sim \eta^{\pi}_h} \left[V^{\projection(\pi)}(s) - \sum_{a \in \Aext} \pi(a|s)Q^{\projection(\pi)}(s,a) \right] \le 0
\end{align}
The inequality is an equality if for any $(s,a)$ s.t. $\eta^{\pi}_h(s)\pi(a|s) > 0$ for some $h$, $\zeta(s,a) = 1$. In another word for any state-action reachable by $\pi$ ($\eta^{\pi}_h(s) >0$ and $\pi(a|s)>0$ for some $h$), $\zeta(s,a) = 1$. 
\end{proof}

The following results shows for any policy $\pi$ in the \weaksettext the projection will not change the policy value much.
\begin{lemma}
\label{lem:constrained_policy_same_value}
For any policy $\pi \in \weakset$, $v^{\pi}_{M} \le v^{\projection(\pi)}_{M'} + \frac{\badstateprob\Vmax}{1-\gamma}$
\end{lemma}
\begin{proof}
Since $\pi$ only takes action in $\Acal$, by Lemma \ref{lem:newmdp_same_value}, we have that $v^{\pi}_{M}=v^{\pi}_{M'}$. Since $\pi \in \weakset$, we have that $\Pr\left( \zeta(s,a)=0 | \pi \right) \le \badstateprob $, which means that:
\begin{align}
    (1-\gamma) \sum_{h=0}^\infty \gamma^h  \EE_{s\sim \eta^{\pi}_h} \left[ \mathds{1} \left( \zeta(s,a) = 0\right) \right] \le \badstateprob
\end{align}

Thus:
\begin{align}
    v^{\projection(\pi)} - v^{\pi} = & \sum_{h=0}^\infty \gamma^h \EE_{s\sim \eta^{\pi}_h} \left[V^{\projection(\pi)}(s) - \sum_{a \in \Aext} \pi(a|s)Q^{\projection(\pi)}(s,a) \right] \\
    = & \sum_{h=0}^\infty \gamma^h \EE_{s\sim \eta^{\pi}_h} \left[V^{\projection(\pi)}(s) - \sum_{a \in \Aext} \pi(a|s)\zeta(s,a) Q^{\projection(\pi)}(s,a) \right] \\
    & - \sum_{h=0}^\infty \gamma^h \EE_{s,a\sim \eta^{\pi}_h} \left[ \mathds{1} \left( \zeta(s,a) = 0\right) Q^{\projection(\pi)}(s,a) \right] \\
    \ge & \sum_{h=0}^\infty \gamma^h \EE_{s\sim \eta^{\pi}_h} \left[V^{\projection(\pi)}(s) - \sum_{a \in \Aext} \pi(a|s)\zeta(s,a) Q^{\projection(\pi)}(s,a) \right] \\
    & - \Vmax \sum_{h=0}^\infty \gamma^h \EE_{s,a\sim \eta^{\pi}_h} \left[ \mathds{1} \left( \zeta(s,a) = 0\right) \right] \\
    \ge & \sum_{h=0}^\infty \gamma^h \EE_{s\sim \eta^{\pi}_h} \left[V^{\projection(\pi)}(s) - \sum_{a \in \Aext} \pi(a|s)\zeta(s,a) Q^{\projection(\pi)}(s,a) \right] - \frac{\Vmax\badstateprob}{1-\gamma} \\
    = & - \frac{\Vmax\badstateprob}{1-\gamma}
\end{align}

The last step follows from the first part in the proof of Lemma \ref{lem:projection_smaller_value}, $v^{\pi}_{M'} -v^{\projection(\pi)}_{M'} \le \frac{\Vmax\badstateprob}{1-\gamma}$.
\end{proof}

\section{Justification of Assumption \ref{asm:bounded_density}}
\label{appendix:assumption}
In this section we prove a claim stated in Section \ref{sec:theory} about the upper bound on density functions. We are going to prove Assumption \ref{asm:bounded_density} holds under when the transition density is bounded.
\begin{lemma}
\label{lem:density_bound} Let $p(\cdot|s,a)$ be the probability density function of transition distribution: $\rho(s_0) \le \sqrt{\densitybound} < \infty, p(s_{t+1}|s_t,a_t) \le \sqrt{\densitybound} < \infty$ and $\forall \pi(a_t|s_t,h) \le \sqrt{\densitybound} < \infty$, for all $s_0, s_t, s_{t+1} \in \Scal $ and $a \in \Acal$. Then in $M'$ for any non-stationary policy $\pi: \Sext \times \mathbb{N} \to \Delta(\Aext)$ and $h \ge 0$, $\eta^{\pi}_h(s,a) \le \densitybound$ for any $s\in \Scal$ and $a \in \Acal$.
\end{lemma}
\begin{proof}
We first prove that $\eta^{\pi}_h(s) \le \sqrt{U}$ for any non-stationary policy $\pi$. For $h=0$, $\eta^{\pi}_h(s) = \rho(s) \le \sqrt{\densitybound}$. For $h \ge 1$ and $s \in \Scal$:
\begin{align}
    \eta^{\pi}_h(s) &= \int_{s_{-1} \in \Sext } \sum_{a \in \Aext } \eta^{\pi}_{h-1}(s_{-1}) \pi(a_{-1}|s_{-1},h-1) p(s|s_{-1},a_{-1}) \mathrm{d} s_{-1} \\
    &= \int_{s_{-1} \in \Scal } \sum_{a \in \Acal } \eta^{\pi}_{h-1}(s_{-1}) \pi(a_{-1}|s_{-1},h-1) p(s|s_{-1},a_{-1}) \mathrm{d} s_{-1} \\
    &\le \mathbb{E}_{\eta^{\pi}_{h-1} \times \pi(h-1) } \left[ p(s|s_{-1},a_{-1}) \right] \\
    &\le \sqrt{\densitybound}
\end{align}
The first step follows from the inductive definition of $\eta^{\pi}_h(s)$. The second step follows from that $s_\text{abs}$ is absorbing state and $a_\text{abs}$ only leads to absorbing state. The third step follows from  transition density $p(s|s_{-1},a_{-1})$ is non-negative. The last step follows from that the transition density $p(s|s_{-1},a_{-1})$ is the same between $M$ and $M'$ for $s, s_{-1} \in \Scal, a_{-1} \in \Acal $, and $p(s|s_{-1},a_{-1})$ in $M$ is upper bounded by $\densitybound$. 
Finally, the joint density function over $s$ and $a$ $\eta^{\pi}_h(s,a) = \eta^{\pi}_h(s)\pi(a|s,h)$ is bounded by $\densitybound$, and we finished the proof.
\end{proof}

For the convenience of notation later we use \emph{admissible distribution} to refer to state-action distributions introduced by non-stationary policy $\pi$ in $M'$. This definition is from \cite{chen2019information}:
\begin{definition}[Admissible distributions] We say a distribution or its density function $\nu \in \Delta(\Sext \times \Aext)$ is admissible in MDP $M'$, if there exists $h \ge 0$, and a (non-stationary) policy $\pi: \Sext \times \mathbb{N} \to \Delta(\Aext)$, such that $\nu (s,a) = \eta^{\pi}_{h}(s,a) $.
\end{definition}

\section{Proofs for Policy Iteration Guarantees}
\label{appendix:api}
In this section we are going to prove the result of Theorem \ref{thm:api} using the definition of the \strongsettext. At a high level, the proof is done in two steps. First we prove similar result to Theorem \ref{thm:api} for any policy in the \strongsettext: an upper bound of $v^{\pi}_{M'} - v^{\pi_t}_{M'}$ where $\pi$ can be any policy in the \strongsettext and $\pi_t$ is the output of the algorithm (Theorem \ref{thm:api_strongset}, formally stated in Appendix \ref{app:api_main_result}). Then we are going to show that for any policy $\pi$ in the \weaksettext after a projection $\projection$ it is in the \strongsettext and $v^{\pi}_M \le v^{\projection(\pi)}_{M'}+ \frac{\Vmax\badstateprob}{1-\gamma}$. Then we can provide the upper bound for $v^{\pi}_{M} - v^{\pi_t}_{M}$ for any $\pi$ in \weaksettext. 

The proof of Theorem \ref{thm:api_strongset} (the $\strongset$ version of Theorem \ref{thm:api}, formally stated in Appendix \ref{app:api_main_result}) goes as follow. First, we show the fixed point of $\OurTpi$ is $Q^{\projection(\pi)}$ for any policy $\pi$, indicating the inner loop of policy evaluation step is actually evaluating $\pi_t = \projection(\widehat{\pi}_t)$. We prove this result formally in Lemma \ref{lem:api_fixed_point}.

To bound the gap between $\pi_t$ and any policy $\picomp$ in the \weaksettext, we use the contraction property of $\OurTpi$ to recursively decompose it into a discounted summation over policy improvement gap $Q^{\pi_{t+1}} - Q^{\pi_t}$. $\picomp$ in the \weaksettext is needed because the operator $\OurTpi$ constrains the backup on the support set of $\zeta$.

Next, we bound the policy improvement gap in Lemma \ref{lem:policy_improvement}:
$$ Q^{\pi_{t+1}} - Q^{\pi_t} \ge -\mathcal{O}(\| \zeta (Q^{\pi_t} - f_{t,K}) \|_{1,\nu} )$$ for some admissible distribution $\nu$ related to $\pi_{t+1}$. The fact that we only need to measure the error on the support set of $\zeta$ is important. It follows from the fact that both $\pi_{t+1}$ and $\pi_{t}$ only takes action on the support set of $\zeta$ except $a_\text{abs}$ which gives us a constant value. This allows us to change the measure from arbitrary distribution $\nu$ to data distribution $\mu$, \emph{without needing concentratability}. 

The rest of proof is to upper bound $\| \zeta (Q^{\pi_t} - f_{t,K}) \|_{1,\nu}$ using contraction and concentration inequalities. First, $\| \zeta (Q^{\pi_t} - f_{t,K}) \|_{1,\nu} $ is upper bounded by $ C \| f_{t,K} - \OurTpi f_{t,K} \|_{2,\mu}/(1-\gamma) $ in Lemma \ref{lem:approximate_policy_value}, using a standard contraction analysis technique. Notice that here we can change the measure to $\mu$ with cost $C$ to allow us to apply concentration inequality. Then Lemma \ref{lem:api_evaluation_accuracy} bounds $\| f_{t,K} - \OurTpi f_{t,K} \|_{2,\mu}$ by a function of sample size $n$ and completeness error $\apicompleteerror$ using Bernstein's inequality.

While writing the proof, we will first introduce the fixed point of $\OurTpi$ is $Q^{\projection(\pi)}$ in section \ref{app:api_proof_fix_point}. We prove the upper bound of the policy evaluation error $\| \zeta (Q^{\pi_t} - f_{t,K}) \|_{1,\nu}$, in section \ref{app:api_policy_evaluation}, and the policy improvement step in section \ref{app:api_policy_improvement}. After we proved the main theorem, we will prove when we can bound the value gap with the optimal value in Corollary \ref{cor:api_optgap}, as we showed in the main text.

\subsection{Fixed point property}
\label{app:api_proof_fix_point}
In Algorithm \ref{alg:api}, the output policy is $\widehat{\pi}_{t+1}$. However, we will show that is actually equivalent with the following algorithm,
\begin{algorithm}[ht]
  \caption{MBS-PI}
  \label{alg:api_same}
    \begin{algorithmic}
  \State {\bfseries Input:} $D$, $\Fcal$, $\Pcal$, $\muhat$, $b$
  \State {\bfseries Output:} $\widehat{\pi}_T$
  \State Initialize $\pi_0 \in \Pcal$.
  \For{$t=0$ {\bfseries to} $T-1$}
  \State Initialize $f_{t,0} \in \Fcal$
  \For{$k=0$ {\bfseries to} $K$} 
  \State //~ \texttt{Policy Evaluation}
  \State $f_{t,k+1} \leftarrow \argmin_{f \in \mathcal{F}} \Lcal_D(f, f_{t,k}; \textcolor{red}{\pi_{t}})$
  \EndFor
  \State //~ \texttt{Policy Improvement}
  \State $\widehat{\pi}_{t+1} \leftarrow \argmax_{\pi \in \Pcal} \mathbb{E}_{D} [\mathbb{E}_{\pi}\left[ \zeta(s,a) f_{t,K}(s,a)\right] ]$
  \State \textcolor{red}{$\pi_{t+1} \leftarrow \projection(\widehat{\pi}_{t+1})$}
  \EndFor
\end{algorithmic}
\end{algorithm}

The output policy is still $\widehat{\pi}_{t+1}$, and we know that $v^{\widehat{\pi}_{t+1}} \ge v^{{\pi}_{t+1}}$. So if we can lower bound $v^{{\pi}_{t+1}}$ we immediately have the lower bound on $v^{\widehat{\pi}_{t+1}}$. The only difference in algorithm is we change the policy evaluation operator from $\OurT^{\widehat{\pi}_t}$ to $\OurT^{\pi_t}$, where ${\pi_t}$ is the projection of $\widehat{\pi}_t$. The following result shows these two operators are actually the same. For the ease of notation, we refer to Algorithm \ref{alg:api_same} in our analysis.

\begin{lemma}
\label{lem:eqvl_after_projection}
For any policy $\pi: \Sext \to \Delta(\Aext)$, $\OurTpi = \OurT^{\projection(\pi)}$.
\end{lemma}
\begin{proof}
We only need to prove for any $f$, $\OurTpi f = \OurT^{\projection(\pi)} f$. For any $a \in \Acal$,
\begin{align}
    (\OurTpi f) (s,a) &= r(s,a) + \gamma \mathbb{E} \left[ \sum_{a' \in \Acal} \pi(a'|s') \zeta(s',a') f(s',a')\right]\\
    &= r(s,a) + \gamma \mathbb{E} \left[ \sum_{a' \in \Acal} \pi(a'|s') \zeta^2(s',a') f(s',a')\right]\\
    &= r(s,a) + \gamma \mathbb{E}_{s'} \left[ \sum_{a' \in \Acal} \projection(\pi_t)(a'|s') \zeta(s',a') Q^{\pi}(s',a')\right] \\
    &= (\OurT^{\projection(\pi)} f)(s,a)
\end{align}
For $a = a_\text{abs}$, $(\OurTpi f) (s,a) = 0 = (\OurT^{\projection(\pi)} f)(s,a)$.
\end{proof}

The next result is a key insight about $\OurTpi$'s behavior in $M'$ that guide our analysis.

\begin{lemma}
\label{lem:api_fixed_point}
For any policy $\pi: \Sext \to \Delta(\Aext)$, the fixed point solution of $\OurTpi$ is equal to $Q^{\projection(\pi)}$ on $\Scal \times \Acal$.
\end{lemma}
\begin{proof}
By definition $Q^{\projection(\pi)}$ is the fixed point of the standard Bellman evaluation operator on $M'$: $\Tcal_{M'}^{\projection(\pi)}$. So for any $(s,a) \in \Scal \times \Acal$:
\begin{align}
    & Q^{\projection(\pi)}(s,a) \\
    &= (\Tcal_{M'}^{\projection(\pi)} Q^{\projection(\pi)})(s,a) \\
    &= r(s,a) + \gamma \mathbb{E}_{s'} \left[ \sum_{a' \in \Aext} \projection(\pi)(a' | s')  Q^{\projection(\pi)}(s',a') \right]\\
    &= r(s,a) + \gamma \mathbb{E}_{s'} \left[ \projection(\pi)(a_\text{abs} | s')  Q^{\projection(\pi)}(s',a_\text{abs}) + \sum_{a' \in \Acal} \projection(\pi)(a' | s')  Q^{\projection(\pi)}(s',a') \right]\\
    &= r(s,a) + \gamma \mathbb{E}_{s'} \left[\sum_{a' \in \Acal} \projection(\pi)(a' | s')  Q^{\projection(\pi)}(s',a') \right]\\
    &= r(s,a) + \gamma \mathbb{E}_{s'}\left[ \sum_{a' \in \Acal} \pi(a' | s') \zeta(s',a') Q^{\projection(\pi)}(s',a')\right]\\
    &=  (\OurTpi Q^{\projection(\pi)})(s,a)
\end{align}
So we proved that $Q^{\projection(\pi)}$ is also the fixed-point solution of $\OurTpi$ constrained on $\Scal \times \Acal$.
\end{proof}
An obvious consequences of these two lemmas is that the fixed point solution of $\OurTpit = \OurT^{\widehat{\pi}_t}$ equals $Q^{\pi_t}$ on $\Scal \times \Acal$.

\subsection{Proofs for policy evaluation step}
\label{app:api_policy_evaluation}
We start with an useful result of the expected loss of the solution from empirical loss minimization, by applying a concentration inequality. 
\begin{lemma}
\label{lem:concenration_2} 
Given $\pi \in \projection(\Pcal)$ and Assumption \ref{asm:api_completeness}, let $g_f^\star=\argmin_{g\in\Fcal}\|g-\OurTpi f\|_{2,\mu}$, then $\|g_f^\star-\OurTpi f\|_{2,\mu}^2\le \avicompleteerror$. The dataset $D$ is generated i.i.d. from $M$ as follows: $(s,a) \sim \muu$, $r = R(s,a)$, $s' \sim P(s,a)$. 
Define $\Lcal_{\mu}(f; f', \pi) = \mathbb{E}_{D} \left[ \Lcal_{D}(f; f', \pi) \right]$. We have that $\forall f\in\Fcal$, with probability at least $1-\delta$,
\begin{align*}
\Lcal_\muu( \OurT_D f; f, \pi)-\Lcal_\muu(g_f^\star; f, \pi) \le \frac{112\Vmax^2\ln\frac{|\Fcal||\Pcal|}{\delta}}{3n}+\sqrt{\frac{64\Vmax^2\ln\frac{|\Fcal||\Pcal|}{\delta}}{n}\avicompleteerror}
\end{align*}
where $\OurTpi_D f = \argmin_{g \in \Fcal} \Lcal_D(g; f, \pi)$.
\end{lemma}

\begin{proof}
    This proof is similar with the proof of Lemma 16 in \cite{chen2019information}, and we adapt it to the $\zeta$-constrained Bellman evaluation operator $\OurTpi$. First, there is no difference in $\Lcal_D$ and $\Lcal_{\mu}$ between $M$ and $M'$, and the right hand side is also the same constant for $M$ and $M'$. The distribution of $D$ in $M$ and $M'$ are the same, since $\mu$ does not cover $s_\text{abs}$ and $a_\text{abs}$. So we are going to prove the inequality for $M$, and thus this bound holds for $M'$ too.
    
    For the simplicity of notations, let $V^{\pi}_{f}(s) = \sum_{a \in \Acal} \pi(a|s) \zeta(s,a) f(s,a)$. Fix any $f, g\in\Fcal$, and define 
    \begin{align}
        X(g,f,g_f^\star) := \left(g(s,a) - r - \gamma V^{\pi}_{f}(s')\right)^2 - \left(g_f^\star(s,a) - r - \gamma V^{\pi}_{f}(s')\right)^2.
    \end{align}
	Plugging each $(s,a,r,s') \in D$ into $X(g,f,g_f^\star)$, we get i.i.d.~variables $X_1(g,f,g_f^\star), X_2(g,f,g_f^\star), \ldots,$ $X_n(g,f,g_f^\star)$. It is easy to see that 
	\begin{align}
	    \frac{1}{n}\sum_{i=1}^n X_i(g,f,g_f^\star) = \Lcal_{D}(g;f,\pi) - \Lcal_{D}(g_f^\star;f,\pi).
	\end{align}
	By the definition of $\Lcal_\mu$, it is also easy to show that
	\begin{align}
	    \Lcal_{\mu}(g; f, \pi) = \left\| g - \OurTpi f \right\|_{2,\mu}^2 +  \mathbb{E}_{s,a \sim \mu} \left[ \mathbb{V}_{r,s'} \left(r + \gamma\sum_{a' \in \Acal} \pi(a'|s')\zeta(s',a')f(s',a') \right) \right],
	\end{align}
    where $\mathbb{V}_{r,s'}$ is the variance over conditional distribution of $r$ and $s'$ given $(s,a)$. Notice that the second part does not depends on $g$. Then
    \begin{align}
        \Lcal_{\muu}(g; f, \pi) - \Lcal_{\muu}(\OurTpi f; f,\pi) = \| g - \OurTpi f  \|_{2,\mu}^2 \label{eq:api_loss_diff}
    \end{align}
	Then we bound the variance of $X$:
	\begin{align}
	\VV[X(g,f,g_f^\star)] \le &~ \EE[X(g,f,g_f^\star)^2] \notag\\
	= &~ \EE_{\mu} \left[\left(\big(g(s,a) - r - \gamma V_{f}(s')\big)^2 - \big(g_f^\star(s,a) - r - \gamma  V_{f}(s')\big)^2\right)^2\right] \tag{Definition of $X$}\\
	= &~ \EE_{\mu} \left[\big(g(s,a)  - g_f^\star(s,a) \big)^2 \big(g(s,a) + g_f^\star(s,a) - 2r - 2\gamma V_{f}(s') \big)^2\right] \notag\\
	\le &~ 4\Vmax^2~ \EE_{\mu} \left[\big(g(s,a)  - g_f^\star(s,a) \big)^2\right] \notag\\
	= &~ 4\Vmax^2~ \|g - g_f^\star\|_{2,\muu}^2 \label{eq:api_bound_var} \notag\\
	\le &~ 8\Vmax^2~ (\EE[X(g,f,g_f^\star)]+2\avicompleteerror).
	\end{align}
	The last step holds because 
	\begin{align*}
	&~\|g-g_f^\star\|_{2,\muu}^2\\
	\leq&~2\left(\|g- \OurTpi f\|_{2,\muu}^2+\|\OurTpi f-g_f^\star\|_{2,\muu}^2\right) \tag{$(a+b)^2\leq 2a^2+2b^2$} \\
	= &~2\left(\|g-\OurTpi f\|_{2,\muu}^2-\|\OurTpi f-g_f^\star\|_{2,\muu}^2+2\|\OurTpi f-g_f^\star\|_{2,\muu}^2\right)\\
	=&~2\left[(\Lcal_{\muu}(g; f,\pi) - \Lcal_{\muu}(\OurTpi f; f,\pi))-(\Lcal_{\muu}(g_f^\star; f,\pi) - \Lcal_{\muu}(\OurTpi f; f,\pi))+2\|\OurTpi f-g_f^\star\|_{2,\muu}^2\right] \tag{Equation (\ref{eq:api_loss_diff})}\\
	=&~2\left[(\Lcal_{\muu}(g; f,\pi) - \Lcal_{\muu}(g_f^\star; f,\pi) +2\|\OurTpi f-g_f^\star\|_{2,\muu}^2\right]\\
	=&~2\left(\EE[X(g,f,g_f^\star)]+2\|\OurTpi f-g_f^\star\|_{2,\muu}^2\right)\\
	\leq &~2(\EE\left[X(g,f,g_f^\star)\right]+2\apicompleteerror)
	\end{align*}
	
	Next, we apply (one-sided) Bernstein's inequality and union bound over all $f\in\Fcal$, $g\in\Fcal$, and $\pi \in \projection(\Pcal)$. With probability at least $1-\delta$, we have
	\begin{align}
	\EE[X(g, f,g_f^\star)] - \frac{1}{n}\sum_{i=1}^n X_i(f,f,g_f^\star)
	\le &~ \sqrt{\frac{2 \VV[X(g,f,g_f^\star)] \ln\tfrac{|\Fcal|^2|\Pcal|}{\delta}}{n}} + \frac{4\Vmax^2 \ln\tfrac{|\Fcal|^2|\Pcal|}{\delta}}{3n} \notag\\
	= &~ \sqrt{\frac{32 \Vmax^2 \left(\EE[X(g, f,g_f^\star)]+2\avicompleteerror\right) \ln\tfrac{|\Fcal||\Pcal|}{\delta}}{n}} + \frac{8\Vmax^2 \ln\tfrac{|\Fcal||\Pcal|}{\delta}}{3n} \label{eq:api_Bernstein}.
	\end{align}
	Since $\OurTpi_D f$ minimizes $\Lcal_D(\holder; f, \pi)$, it also minimizes $\frac{1}{n}\sum_{i=1}^n X_i(\cdot,f,g_f^\star)$. This is because the two objectives only differ by a constant $\Lcal_D(g_f^\star; f,\pi)$. Hence,
	$$
	\frac{1}{n}\sum_{i=1}^n X_i( \OurTpi_D f, f,g_f^\star) \le 
	\frac{1}{n}\sum_{i=1}^n X_i(g_f^\star, f,g_f^\star) = 0.
	$$
	Then,
	\begin{align*}
	\EE[X(\OurTpi_D f,f,g_f^\star)]\le 0 + \sqrt{\frac{32 \Vmax^2 \left(\EE[X(\OurTpi_D, f,g_f^\star)]+2\avicompleteerror\right) \ln\tfrac{|\Fcal||\Pcal|}{\delta}}{n}} + \frac{8\Vmax^2 \ln\tfrac{|\Fcal||\Pcal|}{\delta}}{3n}.
	\end{align*}
	Solving for the quadratic formula,
	\begin{align*}
	\EE[X(\OurTpi_D f, f,g_f^\star)]\le&~ \sqrt{48\left(\frac{8\Vmax^2\ln\frac{|\Fcal||\Pcal|}{\delta}}{3n}\right)^2+
		\frac{64\Vmax^2\ln\frac{|\Fcal||\Pcal| }{\delta}}{n} \avicompleteerror}+\frac{56\Vmax^2\ln\frac{|\Fcal||\Pcal|}{\delta}}{3n}\\
	\leq&~\frac{(56+32\sqrt{3})\Vmax^2\ln\frac{|\Fcal||\Pcal| }{\delta}}{3n}+\sqrt{\frac{64\Vmax^2\ln\frac{|\Fcal||\Pcal| }{\delta}}{n}\avicompleteerror} \tag{$\sqrt{a+b}\leq\sqrt{a}+\sqrt{b}$ and $\ln \frac{|\Fcal|}{\delta}>0$}\\
	\leq&~\frac{112\Vmax^2\ln\frac{|\Fcal||\Pcal| }{\delta}}{3n}+\sqrt{\frac{64\Vmax^2\ln\frac{|\Fcal||\Pcal| }{\delta}}{n}\avicompleteerror}
	\end{align*}
	Noticing that $\EE [X(\OurT_D f, f,g_f^\star)]=\Lcal_\muu(\OurT_D f; f,\pi)-\Lcal_\muu(g_f^\star;f,\pi)$, we complete the proof.
\end{proof}

\begin{lemma}[Policy Evaluation Accuracy]
\label{lem:api_evaluation_accuracy}
For any $t,k \ge 1$ and $\pi_t$, $f_{t,k}$ and $f_{t,k-1}$ from Algorithm \ref{alg:api},
\begin{align*}
    \left\| f_{t,k} -  \OurTpit f_{t,k-1} \right\|_{2,\mu}^2 \le \epsilon_1
\end{align*}
where $\epsilon_1 = \frac{208\Vmax^2\ln\frac{|\Fcal||\Pcal| }{\delta}}{3n} + 2\apicompleteerror$.
\end{lemma}
\begin{proof}
\begin{align*}
    &\left\| f_{t,k} -  \OurTpit f_{t,k-1} \right\|_{2,\mu}^2 \\
    =& \Lcal_{\mu}(f_{t,k};f_{t,k-1},\pi_{t}) - \Lcal_{\mu}(\OurTpit f_{t,k-1};f_{t,k-1},\pi_{t}) \\
    =& \left(\Lcal_{\mu}(f_{t,k};f_{t,k-1},\pi_{t}) - \Lcal_{\mu}(g_{f_{t,k-1}}^\star;f_{t,k-1},\pi_{t}) \right) - \left(\Lcal_{\mu}(\OurTpit f_{t,k-1};f_{t,k-1},\pi_{t}) - \Lcal_{\mu}(g_{f_{t,k-1}}^\star;f_{t,k-1},\pi_{t})\right) \\
    \le&  \frac{112\Vmax^2\ln\frac{|\Fcal||\Pcal| }{\delta}}{3n}+\sqrt{\frac{64\Vmax^2\ln\frac{|\Fcal||\Pcal| }{\delta}}{n}\avicompleteerror} + \left\| g_{f_{t,k-1}}^\star -  \OurTpit f_{t,k-1} \right\|_{2,\mu} \tag{Equation (\ref{eq:api_loss_diff}) and Lemma \ref{lem:concenration_2}} \\
    \le& \frac{112\Vmax^2\ln\frac{|\Fcal||\Pcal| }{\delta}}{3n}+\sqrt{\frac{64\Vmax^2\ln\frac{|\Fcal||\Pcal| }{\delta}}{n}\avicompleteerror} + \apicompleteerror \tag{Definition of $g_{f_{t,k-1}}^\star$ and Assumption \ref{asm:api_completeness}} \\
    \le& \frac{112\Vmax^2\ln\frac{|\Fcal||\Pcal| }{\delta}}{3n}+\frac{32\Vmax^2\ln\frac{|\Fcal||\Pcal| }{\delta}}{n}+\avicompleteerror + \apicompleteerror = \epsilon_1 \tag{$\sqrt{2ab}\le a+b$}
\end{align*}
\end{proof}

From this lemma to the proof of main theorem, we are going to condition on the fact that the event in Assumption \ref{assume:muhat} holds. In the proof of the main theorem we will impose the union bound on all failures.
\begin{lemma}
\label{lem:approximate_policy_value}
For any admissible distribution $\nu$ on $\Sext \times \Aext$, and any $\pi_t$ from Algorithm \ref{alg:api}.
\begin{align}
     \left\|\zeta(s,a) \left( f_{t,K}(s,a) -  Q^{{\pi}_t}(s,a) \right) \right\|_{1,\nu} &\le \frac{ C \left(\sqrt\epsilon_1 +\Vmax\tvmuerror\right)  }{1-\gamma} + \gamma^K \Vmax
\end{align}
where $\epsilon_1$ is defined in Lemma \ref{lem:api_evaluation_accuracy}.
\end{lemma}
(Although $f_{t,K}$ is only defined on $\Scal \times \Acal$, $\zeta$ is always zero for any other $(s,a)$. Thus the all values used in the proof are well-defined. Later, when it is necessary for proof, we define the value of $f_{t,K}$ outside of $\Scal \times \Acal$ to be zero. In the algorithm, we will never need to query the value of $f_{t,K}$ outside of $\Scal \times \Acal$.)
\begin{proof}
For any $k \ge 1$ and any distribution $\nu$ on $\Sext \times \Aext$:
\begin{align}
    & \left\|\zeta \left( f_{t,k} -  Q^{\pi_t} \right) \right\|_{1,\nu} \\
    &\le \left\| \zeta\left( f_{t,k} -  \OurTpit f_{t,k-1} \right)\right\|_{1,\nu} + \left\| \zeta \left( \OurTpit f_{t,k-1} -  \OurTpit Q^{\pi_t} \right) \right\|_{1,\nu} \\
    &\le \left\|\zeta\left( f_{t,k} - \OurTpit f_{t,k-1} \right)\right\|_{1,\nu} + \left\| \OurTpit f_{t,k-1} -  \OurTpit Q^{\pi_t} \right\|_{1,\nu} \\
    &\le C \left\|f_{t,k} -  \OurTpit f_{t,k-1} \right\|_{1,\muhat} + \left\|  \OurTpit f_{t,k-1} -  \OurTpit Q^{\pi_t} \right\|_{1, \nu} \label{eq:api_change_measure} \\
    &\le  C \left( \left\|f_{t,k} -  \OurTpit f_{t,k-1} \right\|_{1,\mu} + \Vmax \tvmuerror \right) + \left\|  \OurTpit f_{t,k-1} -  \OurTpit Q^{\pi_t} \right\|_{1,\nu} \label{eq:api_tv_change_measure} \\
    &\le C \left( \left\|f_{t,k} -  \OurTpit f_{t,k-1} \right\|_{2,\mu} +\Vmax\tvmuerror\right) +  \left\| \OurTpit f_{t,k-1} -  \OurTpit Q^{\pi_t}  \right\|_{1,\nu} \tag{Jensen's inequality}\\
    &\le C (\sqrt\epsilon_1 +\Vmax\tvmuerror) +  \left\| \OurTpit f_{t,k-1} -  \OurTpit Q^{\pi_t}  \right\|_{1,\nu} \tag{Lemma \ref{lem:api_evaluation_accuracy}}\\
    &= C (\sqrt\epsilon_1 +\Vmax\tvmuerror)  +   \mathbb{E}_{\nu} \left| \gamma \mathbb{E}_{P(\nu)} \sum_{a'\in \Acal} \pi_t(a'|s') \zeta(s',a') \left( f_{t,k-1}(s',a') - Q^{\pi_t}(s',a') \right) \right| \label{eq:policy_eval_property} \\
    &= C (\sqrt\epsilon_1 +\Vmax\tvmuerror) +   \mathbb{E}_{\nu} \left[ \gamma \mathbb{E}_{P(\nu) \times \pi_t}  \left| \zeta(s',a') \left(  f_{t,k-1}(s',a') - Q^{\pi_t}(s',a') \right)\right|  \right] \label{eq:write_as_expectation} \\
    &\le C (\sqrt\epsilon_1 +\Vmax\tvmuerror) +  \gamma \mathbb{E}_{P(\nu)\times\pi_t} \left| \zeta(s',a') \left( f_{t,k-1}(s',a') - Q^{\pi_t}(s',a') \right) \right| \\
    &\le C (\sqrt\epsilon_1 +\Vmax\tvmuerror) + \gamma \left\| \zeta \left( f_{t,k-1} -  Q^{\pi_t} \right) \right\|_{1,P(\nu)\times\pi}
\end{align}
Equation (\ref{eq:api_change_measure}) holds since for all $(s,a)$ s.t. $\zeta(s,a) > 0$, $\nu(s,a) \le \densitybound \le \frac{\densitybound}{b} \muhat(s,a) = C\muhat(s,a)$. Equation (\ref{eq:api_tv_change_measure}) holds since the total variation distance between $\mu$ and $\muhat$ is bounded by $\tvmuerror$ and the Bellman error is bounded in $[-\Vmax, \Vmax]$. Equation (\ref{eq:policy_eval_property}) follows from $\pi_t \in \strongset$. So if $\zeta(s,a) = 0$, $\pi(a|s) = 0$ for all $a \in \Acal$. Equation (\ref{eq:write_as_expectation}) holds since $\zeta(\cdot,a_\text{abs}) = 0$. The next equation follows from that $\zeta = \zeta^2$.

Note that this holds for any admissible distribution $\nu$ on $\Sext \times \Aext$ and and $k$, as well as $\epsilon_1$ does not depends on $k$. Repeating this for $k$ from $K$ to 1 we will have that
\begin{align}
     \left\|\zeta(s,a) \left( f_{t,K}(s,a) -  Q^{{\pi}_t}(s,a) \right) \right\|_{1,\nu} \le& \frac{1-\gamma^K}{1-\gamma}C \left(\sqrt\epsilon_1 +\Vmax\tvmuerror\right)  + \gamma^K \Vmax \\
     <& \frac{ C \left(\sqrt\epsilon_1 +\Vmax\tvmuerror\right)  }{1-\gamma} + \gamma^K \Vmax
\end{align}
\end{proof}

\subsection{Proofs for policy improvement step}
\label{app:api_policy_improvement}
\begin{lemma}[Concentration of Policy Improvement Loss] 
\label{lem:concenration_policy_improvement}
For any $f \in \Fcal$, with probability at least $1-\delta$,
$$ \left\| \mathbb{E}_{\widehat{\pi}_{f}} \left[ \zeta(s,a) f(s, a)\right] - \max_{a \in \mathcal{A}}\zeta(s,a) f(s,a) \right\|_{1,\mu}  \le \apigreedyerror + 2\Vmax\sqrt{ \frac{\ln(|\Fcal||\Pcal|/\delta)}{2n}}$$
where $\widehat{\pi}_f = \argmax_{\pi \in \Pcal} \EE_{D} \left[  \mathbb{E}_{\pi}\left[ \zeta(s,a) f(s,a)\right] \right]$.
\end{lemma}
\begin{proof} 
Fixed $f$, define $X(s;\pi) = \max_{a \in \mathcal{A}}\zeta(s,a) f(s,a) - \mathbb{E}_{\pi}\left[ \zeta(s,a) f(s,a)\right] $. Notice that by definition $X(s;\pi)$ is always non-negative, and  $\widehat{\pi}_f = \argmax_{\pi \in \Pcal} \EE_{D} \left[  \mathbb{E}_{\pi}\left[ \zeta(s,a) f(s,a)\right] \right] = \argmin_{\pi \in \Pcal} \EE_{D}[ X(s;\pi) ].$

Only in this proof, let $\pi_f$ be:
$$\argmin_{\pi \in \Pcal} \EE_{\mu}[ X(s;\pi) ] = \argmin_{\pi \in \Pcal} \left\| \mathbb{E}_{{\pi}} \left[ \zeta(s,a) f(s, a)\right] - \max_{a \in \mathcal{A}}\zeta(s,a) f(s,a) \right\|_{1,\mu}.$$

$X(s;\pi) \in [0,\Vmax]$. By Hoeffding's inequality and union bound over all $\pi \in \Pcal$, $f \in \Fcal$, with probability at least $1-\delta$ for any $f$ and $\pi \neq \pi_f$,
\begin{align}
    \EE_{\mu}[ X(s;\pi) ]  - \EE_{D}[ X(s;\pi) ] \le \Vmax\sqrt{ \frac{\ln(|\Fcal||\Pcal|/\delta)}{2n}}
\end{align}
for $\pi=\pi_f$
\begin{align}
    \EE_{D}[ X(s;\pi) ]  - \EE_{\mu}[ X(s;\pi) ] \le \Vmax\sqrt{ \frac{\ln(|\Fcal||\Pcal|/\delta)}{2n}}
\end{align}
If $\widehat{\pi}_f = \pi_f$, then $\EE_{\mu}[ X(s;\widehat{\pi}_f) ] \le \apigreedyerror$. Otherwise,
\begin{align}
    &\EE_{\mu}[ X(s;\widehat{\pi}_f) ] \\
    \le~&  \EE_{D}[ X(s;\widehat{\pi}_f) ] + \Vmax\sqrt{ \frac{\ln(|\Fcal||\Pcal|/\delta)}{2n}} \\
    \le~&  \EE_{D}[ X(s;{\pi}_f) ] + \Vmax\sqrt{ \frac{\ln(|\Fcal||\Pcal|/\delta)}{2n}} \\
    \le~& \EE_{\mu}[ X(s;{\pi}_f) ] + 2\Vmax\sqrt{ \frac{\ln(|\Fcal||\Pcal|/\delta)}{2n}} \\
    = ~& \min_{\pi \in \Pcal} \left\| \mathbb{E}_{\widehat{\pi}} \left[ \zeta(s,a) f(s, a)\right] - \max_{a \in \mathcal{A}}\zeta(s,a) f(s,a) \right\|_{1,\mu} + 2\Vmax\sqrt{ \frac{\ln(|\Fcal||\Pcal|/\delta)}{2n}} \\
    = ~& \apigreedyerror + 2\Vmax\sqrt{ \frac{\ln(|\Fcal||\Pcal|/\delta)}{2n}}
\end{align}
\end{proof}
For the following proof until the main theorem, we are going to condition on the fact that the high probability bound in the lemma above holds, and impose an union bound in the proof of main theorem.
\begin{lemma}
\label{lem:greedy_improvement}
For any admissible distribution $\nu$ on $\Sext$, any policy $\pi: \Sext \to \Delta(\Aext)$,
\begin{align*}
    \mathbb{E}_{\nu}\left[ \mathbb{E}_{\pi_{t+1}} \left[ \zeta(s,a) f_{t,K}(s, a)\right] - \mathbb{E}_{\pi} \left[\zeta(s,a) f_{t,K}(s, a) \right]\right] \ge \\
- C \left(\apigreedyerror+\Vmax\tvmuerror + 2\Vmax\sqrt{ \frac{\ln(|\Fcal||\Pcal|/\delta)}{2n}}\right)
\end{align*}
\end{lemma}
\begin{proof}
Recall that $\pi_{t+1} = \projection(\widehat{\pi}_{t+1})$. So $\pi_{t+1}(a|s) = \widehat{\pi}_{t+1}(a|s)$ for all $a$ such that $\zeta(s,a)=1$. Then 
\begin{align*}
    \mathbb{E}_{\pi_{t+1}} \left[ \zeta(s,a) f_{t,K}(s, a)\right] &= \mathbb{E}_{\widehat{\pi}_{t+1}} \left[ \zeta(s,a) f_{t,K}(s, a)\right]\\
    \mathbb{E}_{\nu}\left[ \mathbb{E}_{\pi_{t+1}} \left[ \zeta(s,a) f_{t,K}(s, a)\right] \right] &= \mathbb{E}_{\nu}\left[ \mathbb{E}_{\widehat{\pi}_{t+1}} \left[ \zeta(s,a) f_{t,K}(s, a)\right] \right]
\end{align*}

\begin{align}
    & \mathbb{E}_{\nu}\left[ \mathbb{E}_{\pi_{t+1}} \left[ \zeta(s,a) f_{t,K}(s, a)\right] - \mathbb{E}_{\pi} \left[\zeta(s,a) f_{t,K}(s, a) \right]\right] \\
    &= \mathbb{E}_{\nu}\left[ \mathbb{E}_{\widehat{\pi}_{t+1}} \left[ \zeta(s,a) f_{t,K}(s, a)\right] - \mathbb{E}_{\pi} \left[\zeta(s,a) f_{t,K}(s, a) \right]\right] \\
    &= \mathbb{E}_{\nu}\left[ \mathbb{E}_{\widehat{\pi}_{t+1}} \left[ \zeta(s,a) f_{t,K}(s, a)\right] - \max_{a \in \mathcal{A}}\zeta(s,a) f_{t,K}(s,a) + \max_{a \in \mathcal{A}}\zeta(s,a) f_{t,K}(s,a) - \mathbb{E}_{\pi} \left[\zeta(s,a) f_{t,K}(s, a) \right]\right] \\
    &\ge \mathbb{E}_{\nu}\left[ \mathbb{E}_{\widehat{\pi}_{t+1}} \left[ \zeta(s,a) f_{t,K}(s, a)\right] - \max_{a \in \mathcal{A}}\zeta(s,a) f_{t,K}(s,a) \right] \\
    &\ge - \mathbb{E}_{\nu}\left| \mathbb{E}_{\widehat{\pi}_{t+1}} \left[ \zeta(s,a) f_{t,K}(s, a)\right] - \max_{a \in \mathcal{A}}\zeta(s,a) f_{t,K}(s,a) \right|\\
    &= - \left\| \mathbb{E}_{\widehat{\pi}_{t+1}} \left[ \zeta(s,a) f_{t,K}(s, a)\right] - \max_{a \in \mathcal{A}}\zeta(s,a) f_{t,K}(s,a) \right\|_{1,\nu}\\
    &\ge - C\left\| \mathbb{E}_{\widehat{\pi}_{t+1}} \left[ \zeta(s,a) f_{t,K}(s, a)\right] - \max_{a \in \mathcal{A}}\zeta(s,a) f_{t,K}(s,a) \right\|_{1,\muhat}
\end{align}
The last step follows from that $\zeta(s,a)=1 \Rightarrow \muhat(s,a) \ge b \Rightarrow \muhat(s) \ge b \Rightarrow -\nu(s) \ge -\densitybound \ge -C\muhat(s) $, and for all other $(s,a)$ the term inside of norm is zero. Since the total variation distance between $\muhat$ and $\mu$ is bounded by $\tvmuerror$
\begin{align}
 &\left\| \mathbb{E}_{\widehat{\pi}_{t+1}} \left[ \zeta(s,a) f_{t,K}(s, a)\right] - \max_{a \in \mathcal{A}}\zeta(s,a) f_{t,K}(s,a) \right\|_{1,\muhat} \\
 \le~& \left\| \mathbb{E}_{\widehat{\pi}_{t+1}} \left[ \zeta(s,a) f_{t,K}(s, a)\right] - \max_{a \in \mathcal{A}}\zeta(s,a) f_{t,K}(s,a) \right\|_{1,\mu} + \Vmax\tvmuerror
\end{align}
By Lemma \ref{lem:concenration_policy_improvement}:
\begin{align}
    &\left\| \mathbb{E}_{\widehat{\pi}_{t+1}} \left[ \zeta(s,a) f_{t,K}(s, a)\right] - \max_{a \in \mathcal{A}}\zeta(s,a) f_{t,K}(s,a) \right\|_{1,\mu} \le \apigreedyerror + 2\Vmax\sqrt{ \frac{\ln(|\Fcal||\Pcal|/\delta)}{2n}}
\end{align}
Then we finished the proof by plug this into the last equation.
\end{proof}

\begin{lemma}
\label{lem:policy_improvement}
For any $(s,a) \in \Sext \times \Aext$, and any $\pi_t$, $\pi_{t+1}$ in Algorithm \ref{alg:api},
\begin{align}
    Q^{\pi_{t+1}}(s,a) -  Q^{\pi_{t}}(s,a) \ge -\frac{ 2C\sqrt{\epsilon_1}+3\Vmax C \tvmuerror}{(1-\gamma)^2}  -\frac{\epsilon_2 + 2\gamma^K \Vmax}{1-\gamma}
\end{align}
where $\epsilon_1$ is defined in Lemma \ref{lem:api_evaluation_accuracy}, $\epsilon_2 = C \left(\apigreedyerror+ 2\Vmax\sqrt{ \frac{\ln(|\Fcal||\Pcal|/\delta)}{2n}}\right)$.
\end{lemma}
\begin{proof}
For any $s'$, only in this proof, let $\eta^{\pi_{t+1}}_h$ be the state distribution on the $h$th step from initial state $s'$ following $\pi_{t+1}$. By applying performance difference lemma \cite{kakade2002approximately},
\begin{align}
    &V^{{\pi}_{t+1}}(s') - V^{{\pi}_{t}}(s') \\
    &= \sum_{h=1}^\infty \gamma^{h-1} \EE_{z\sim \eta^{\pi_{t+1}}_h}\left[ \sum_{a \in \Aext} \left( {\pi}_{t+1}(a|z) Q^{{\pi}_{t}}(z, a) -  {\pi}_{t}(a|z) Q^{{\pi}_{t}}(z, a) \right) \right] \\
    &= \sum_{h=1}^\infty \gamma^{h-1} \EE_{z\sim \eta^{\pi_{t+1}}_h} \left[ \sum_{a \in \Aext} (1-\zeta(z,a)) \left( {\pi}_{t+1}(a|z) Q^{{\pi}_{t}}(z, a) -  {\pi}_{t}(a|z) Q^{{\pi}_{t}}(z, a) \right) \right. \\
    & \quad \left. + \sum_{a \in \Aext} \zeta(z,a) \left( {\pi}_{t+1}(a|z) Q^{{\pi}_{t}}(z, a) -  {\pi}_{t}(a|z) Q^{{\pi}_{t}}(z, a) \right) \right] 
\end{align}
Because $\pi_t, \pi_{t+1} \in \strongset$, $\zeta(z,a) = 0$ means either ${\pi}_{t}(a|z) = {\pi}_{t+1}(a|z) = 0$ or $a = a_\text{abs}$. So the first term is zero. Then:
\begin{align}
    &V^{\pi_{t+1}}(s') - V^{\pi_{t}}(s') \\
    &= \sum_{h=1}^\infty \gamma^{h-1} \EE_{z\sim \eta^{\pi_{t+1}}_h} \left[ \sum_{a \in \Aext} \zeta(z,a) \left( \pi_{t+1}(a|z) Q^{\pi_{t}}(z, a) - \pi_{t}(a|z) Q^{\pi_{t}}(z, a) \right) \right]  \\
    &= \sum_{h=1}^\infty \gamma^{h-1} \EE_{z\sim \eta^{\pi_{t+1}}_h} \left[ \sum_{a \in \Acal} \zeta(z,a) \left( \pi_{t+1}(a|z) Q^{\pi_{t}}(z, a) - \pi_{t}(a|z) Q^{\pi_{t}}(z, a) \right) \right] \label{eq:api_value_of_aabs_is_zero} \\
    &= \sum_{h=1}^\infty \gamma^{h-1} \EE_{z\sim \eta^{\pi_{t+1}}_h} \left[ \sum_{a \in \Acal} \zeta(z,a) \left( \pi_{t+1}(a|z) Q^{{\pi}_{t}}(z, a) -  \pi_{t+1}(a|z) f_{t,K}(z, a) \right) \right.  \\
    & \quad + \left. \sum_{a \in \Acal} \zeta(z,a) \left( \pi_{t+1}(a|z) f_{t,K}(z, a) -  \pi_t(a|z) f_{t,K}(z, a) \right) \right. \label{eq:greedy_improvement} \\
    & \quad + \left. \sum_{a \in \Acal} \zeta(z,a) \left( \pi_t(a|z) f_{t,K}(z, a) -  \pi_t(a|z) Q^{\pi_{t}}(z, a) \right) \right]
\end{align}
Equation \ref{eq:api_value_of_aabs_is_zero} follows from $Q^{\pi}(s,a_\text{abs}) = 0$ for any $\pi$ and $s$. By Lemma \ref{lem:greedy_improvement}, for any $h$,
\begin{align}
    &\EE_{z\sim \eta^{\pi_{t+1}}_h} \left[\sum_{a \in \Acal} \zeta(z,a) \left( \pi_{t+1}(a|z) f_{t,K}(z, a) -  \pi_t(a|z) f_{t,K}(z, a) \right) \right] \\
    &= \EE_{z\sim \eta^{\pi_{t+1}}_h}\left[ \mathbb{E}_{\pi_{t+1}} \left[ \zeta(s,a) f_{t,K}(s, a)\right] - \mathbb{E}_{\pi_t} \left[\zeta(s,a) f_{t,K}(s, a) \right]\right] \ge - \epsilon_2 - C\Vmax\tvmuerror 
\end{align}
Then
\begin{align}
     & V^{\pi_{t+1}}(s') - V^{\pi_{t}}(s') \\ &\ge \sum_{h=1}^\infty \gamma^{h-1} \EE_{z\sim \eta^{\pi_{t+1}}_h}  \left[ \sum_{a \in \Acal} \zeta(z,a) \left( \pi_{t+1}(a|z) Q^{{\pi}_{t}}(z, a) -  \pi_{t+1}(a|z) f_{t,K}(z, a) \right) \right.  \\
    & \quad + \left. \sum_{a \in \Acal} \zeta(z,a) \left( \pi_t(a|z) f_{t,K}(z, a) -  \pi_t(a|z) Q^{\pi_{t}}(z, a) \right) \right]-\frac{\epsilon_2 + C\Vmax\tvmuerror }{1-\gamma} \\
    &\ge - \sum_{h=1}^\infty \gamma^{h-1} \left( \left\| \zeta(z,a)(Q^{{\pi}_{t}}(z, a) - f_{t,K}(z, a)) \right\|_{1,\eta^{\pi_{t+1}}_h} \right. \\
    & \quad \left. + \left\| \zeta(z,a)(Q^{{\pi}_{t}}(z, a) - f_{t,K}(z, a)) \right\|_{1,\eta^{\pi_{t+1}}_h \times \pi_{t}} \right) -\frac{\epsilon_2 + C\Vmax\tvmuerror}{1-\gamma} \\
    &\ge - \sum_{h=1}^\infty \gamma^{h-1} \left( \left\| \zeta(z,a)(Q^{{\pi}_{t}}(z, a) - f_{t,K}(z, a)) \right\|_{2,\eta^{\pi_{t+1}}_h} \right. \\
    & \quad + \left. \left\| \zeta(z,a)(Q^{{\pi}_{t}}(z, a) - f_{t,K}(z, a)) \right\|_{2,\eta^{\pi_{t+1}}_h \times \pi_{t}} \right) -\frac{\epsilon_2 + C\Vmax\tvmuerror}{1-\gamma} \label{eq:api_l1_to_l2}\\
    &\ge \frac{- 2C \left(\sqrt{\epsilon_1}+\Vmax\tvmuerror\right) }{(1-\gamma)^2} - \frac{2\gamma^K \Vmax}{1-\gamma} -\frac{\epsilon_2 + C\Vmax\tvmuerror}{1-\gamma} \tag{Lemma \ref{lem:approximate_policy_value}}
\end{align}
Equation \ref{eq:api_l1_to_l2} follows from Jensen's inequality. Since this holds for any $s'$, we proved that for any $(s,a)$,
\begin{align}
    &\left[  Q^{\pi_{t+1}}(s,a) -  Q^{\pi_{t}}(s,a) \right] \\
    =~& \gamma \mathbb{E}_{s'} \left[ V^{\pi_{t+1}}(s') - V^{\pi_{t}}(s') \right] \\
    \ge~& \frac{- 2C \left(\sqrt{\epsilon_1}+\Vmax\tvmuerror\right) }{(1-\gamma)^2} - \frac{2\gamma^K \Vmax}{1-\gamma} -\frac{\epsilon_2 + C\Vmax\tvmuerror}{1-\gamma} \\
    \ge~& -\frac{ 2C\sqrt{\epsilon_1}+3C\Vmax\tvmuerror }{(1-\gamma)^2} - \frac{2\gamma^K \Vmax}{1-\gamma} -\frac{\epsilon_2}{1-\gamma}
\end{align}
\end{proof}

\subsection{Proof of main theorems}
\label{app:api_main_result}
\begin{theorem}
\label{thm:api_strongset}
Given an MDP $M = <\Scal, \Acal, R, P, \gamma, p>$, a dataset $D = \{ (s,a,r,s') \}$ with $n$ samples that is draw i.i.d. from $\mu \times R \times P$, and a finite Q-function classes $\Fcal$ and a finite policy class $\Pcal$ satisfying Assumption \ref{asm:api_completeness} and \ref{asm:weak_api_pi_realizalibity}, $\pi_{t} = \projection(\widehat{\pi}_t)$ from Algorithm \ref{alg:api} satisfies that with probability at least $1-3\delta$,
\begin{align*}
    v^{\picomp} - v^{\pi_t} \le \frac{4C}{(1-\gamma)^3}\left(\sqrt{\frac{419\Vmax^2\ln\frac{|\Fcal||\Pcal| }{\delta}}{3n}} + 2\sqrt{\apicompleteerror}\right) + \frac{6C\Vmax\tvmuerror}{(1-\gamma)^3} + \frac{2C\apigreedyerror + 3\gamma^{K-1} \Vmax}{(1-\gamma)^2}
\end{align*}
for any policy $\picomp \in \strongset$.
\end{theorem}

\begin{proof}
For simplicity of the notation, let $\epsilon_1 = \frac{208\Vmax^2\ln\frac{|\Fcal||\Pcal| }{\delta}}{3n} + 2\apicompleteerror$, $\epsilon_2 = C \left(\apigreedyerror+ 2\Vmax\sqrt{ \frac{\ln(|\Fcal||\Pcal|/\delta)}{2n}}\right)$
and $\epsilon_3 = \frac{2C\sqrt{\epsilon_1 }+3\Vmax C\tvmuerror}{(1-\gamma)^2} + \frac{\epsilon_2 + 2\gamma^K \Vmax}{1-\gamma}$. We start by proving a stronger result. For any $\picomp \in \strongset$, we will upper bound $\E_\nu[ V^{\picomp} - V^{\pi_{t}}]$ for any admissible distribution $\nu$ over $\Sext$ which will naturally be an upper bound for $v^{\picomp} - v^{\pi_t}$
\begin{align*}
    &\E_\nu[ V^{\picomp} - V^{\pi_{t+1}}] \\
    &= \E_{\nu} \left[V^{\picomp}(s) - \sum_{a \in \Aext} \pi_{t+1}(a|s) Q^{\pi_t}(s,a) +  \sum_{a \in \Aext} \pi_{t+1}(a|s) Q^{\pi_t}(s,a) - V^{\pi_{t+1}}(s)\right]\\
    &= \E_{\nu} \left[V^{\picomp}(s) - \sum_{a \in \Aext} \pi_{t+1}(a|s) Q^{\pi_t}(s,a) +  \sum_{a \in \Aext} \pi_{t+1}(a|s) \left( Q^{\pi_t}(s,a) - Q^{\pi_{t+1}}(s,a)\right)\right]\\
    &\leq \E_{\nu}\sum_{a \in \Aext} \left[\picomp(a|s) Q^{\picomp}(s,a) - \pi_{t+1}(a|s) Q^{\pi_t}(s,a)\right] + \epsilon_3 \tag{Lemma \ref{lem:policy_improvement}} \\
    &= \E_{\nu}\sum_{a \in \Aext} \zeta(s,a)[\picomp(a|s) Q^{\picomp}(s,a) - \pi_{t+1}(a|s) Q^{\pi_t}(s,a)] + \epsilon_3 \\
    &= \E_{\nu} \left[ \mathbb{E}_{\picomp} \left[ \zeta(s,a) Q^{\picomp}(s,a) \right] - \mathbb{E}_{\pi_{t+1}} \left[ \zeta(s,a) f_t(s,a) \right] \right. \\
    &\quad \left.+ \mathbb{E}_{\pi_{t+1}} \left[ \zeta(s,a) f_t(s,a) \right] - \mathbb{E}_{\pi_{t+1}} \left[ \zeta(s,a) Q^{\pi_{t+1}}(s,a) \right] \right]  + \epsilon_3\\
    &\le \E_{\nu} \left[ \mathbb{E}_{\picomp} \left[ \zeta(s,a) Q^{\picomp}(s,a) \right] - \mathbb{E}_{\pi_{t+1}} \left[ \zeta(s,a) f_t(s,a) \right] \right] \\
    &\quad + \left\| \zeta(z,a)(Q^{{\pi}_{t}}(z, a) - f_t(z, a)) \right\|_{1,\nu\times\pi_{t+1}} + \epsilon_3 \\
    &\le \E_{\nu} \left[ \mathbb{E}_{\picomp} \left[ \zeta(s,a) Q^{\picomp}(s,a) \right] - \mathbb{E}_{\pi_{t+1}} \left[ \zeta(s,a) f_t(s,a) \right] \right]  + \frac{C\sqrt{ \epsilon_1} + C\Vmax\tvmuerror}{1-\gamma} + \gamma^K \Vmax + \epsilon_3 \tag{Lemma \ref{lem:approximate_policy_value}}\\
    &\le \E_{\nu} \left[ \mathbb{E}_{\picomp} \left[ \zeta(s,a) Q^{\picomp}(s,a) \right] - \mathbb{E}_{\picomp} \left[ \zeta(s,a) f_t(s,a) \right] \right] + \epsilon_2 + C\Vmax\tvmuerror + \frac{C\sqrt{ \epsilon_1} + C\Vmax\tvmuerror}{1-\gamma} + \gamma^K \Vmax + \epsilon_3 \tag{Lemma \ref{lem:greedy_improvement}}\\
    &\le \E_{\nu} \left[ \mathbb{E}_{\picomp} \left[ \zeta(s,a) Q^{\picomp}(s,a) \right] - \mathbb{E}_{\picomp} \left[ \zeta(s,a) Q^{\pi_{t}}(s,a) \right] \right] + \epsilon_2 + \frac{2C\sqrt{ \epsilon_1} + 3C\Vmax\tvmuerror}{1-\gamma} + 2\gamma^K \Vmax + \epsilon_3 \tag{Lemma \ref{lem:approximate_policy_value}}\\
    &= \E_{\nu \times \picomp} \left[ \zeta(s,a) Q^{\picomp}(s,a) -  \zeta(s,a) Q^{\pi_{t}}(s,a)  \right] + \epsilon_2 + \frac{2C\sqrt{ \epsilon_1} + 3C\Vmax\tvmuerror}{1-\gamma} + 2\gamma^K \Vmax + \epsilon_3\\
    &= \E_{\nu \times \picomp} \left[ Q^{\picomp}(s,a) -  Q^{\pi_{t}}(s,a)  \right] + \epsilon_2 + \frac{2C\sqrt{ \epsilon_1} + 3C\Vmax\tvmuerror}{1-\gamma} + 2\gamma^K \Vmax + \epsilon_3 \tag{$\pi_t \in \strongset$} \\
    &\leq \gamma\E_{ P(\nu \times \picomp)} [V^{\picomp} - V^{\pi_t}] + \epsilon_2 + \frac{2C\sqrt{ \epsilon_1} + 3C\Vmax\tvmuerror}{1-\gamma} + 2\gamma^K \Vmax + \epsilon_3 
\end{align*}
The second to last step follows from $\pi_t \in \strongset$: for all $s,a$ such that $\picomp (a|s)>0$, either $\zeta(s,a) = 1$, or $a=a_{\text{abs}}$. The later two indicate that $Q^{\pi_t}(s,a) = Q^{\picomp}(s,a) = 0$. So for all $s,a$ such that $\picomp (a|s)>0$, $Q^{\picomp}(s,a) = \zeta(s,a)Q^{\picomp}(s,a)$ and $Q^{\pi_t}(s,a) = \zeta(s,a)Q^{\pi_t}(s,a)$. Now we proved 
\begin{align}
    \E_\nu[ V^{\picomp} - V^{\pi_{t+1}}] \leq \gamma\E_{ P(\nu \times \picomp)} [V^{\picomp} - V^{\pi_t}] + \epsilon_2 + \epsilon_3 + \frac{2C\sqrt{ \epsilon_1} + 3C\Vmax\tvmuerror}{1-\gamma} + 2\gamma^K \Vmax 
\end{align}
holds for any distribution $\nu$. The error terms do not depend on $t$ and this holds for any $t$. We can repeatedly apply this for all $0<t' \le t$. Assuming $t \ge K$ this will give us :
\begin{align*}
    &\E_\nu[ V^{\picomp} - V^{\pi_{t+1}}] \\
    \le& \frac{1-\gamma^t}{1-\gamma}\left(\epsilon_2 + \epsilon_3 + \frac{2C\sqrt{ \epsilon_1} + 3C\Vmax\tvmuerror}{1-\gamma} + 2\gamma^K \Vmax \right) + \gamma^t \Vmax \\
    \le& \frac{\epsilon_2}{1-\gamma} + \frac{\epsilon_3}{1-\gamma} + \frac{2C\sqrt{ \epsilon_1}}{(1-\gamma)^2} + \frac{3C\Vmax\tvmuerror}{(1-\gamma)^2} + \frac{3\gamma^K \Vmax}{1-\gamma}  \\
    \le& \frac{2\epsilon_2}{(1-\gamma)^2} + \frac{4C\sqrt{ \epsilon_1}}{(1-\gamma)^3} + \frac{6C\Vmax\tvmuerror}{(1-\gamma)^3} + \frac{3\gamma^{K-1} \Vmax}{(1-\gamma)^2}  \\
    \le& \frac{2C\apigreedyerror}{(1-\gamma)^2} + \frac{4C}{(1-\gamma)^2}\sqrt{\frac{\Vmax^2\ln(|\Fcal||\Pcal|/\delta)}{2n}} + \frac{4C\sqrt{ \epsilon_1}}{(1-\gamma)^3} + \frac{6C\Vmax\tvmuerror}{(1-\gamma)^3} + \frac{3\gamma^{K-1} \Vmax}{(1-\gamma)^2}  \\
    \le& \frac{2C\apigreedyerror}{(1-\gamma)^2} + \frac{4C}{(1-\gamma)^3} \left( \sqrt{\frac{\Vmax^2\ln(|\Fcal||\Pcal|/\delta)}{2n}} + \sqrt{ \frac{208\Vmax^2\ln(|\Fcal||\Pcal|/\delta)}{3n} + 2\apicompleteerror }\right) \\
    &+ \frac{6C\Vmax\tvmuerror}{(1-\gamma)^3} + \frac{3\gamma^{K-1} \Vmax}{(1-\gamma)^2}  \\
    \le& \frac{2C\apigreedyerror}{(1-\gamma)^2} + \frac{4C}{(1-\gamma)^3} \left( \sqrt{\frac{\Vmax^2\ln(|\Fcal||\Pcal|/\delta)}{2n}} + \sqrt{ \frac{208\Vmax^2\ln(|\Fcal||\Pcal|/\delta)}{3n} } + \sqrt{2\apicompleteerror} \right) \\
    &+ \frac{6C\Vmax\tvmuerror}{(1-\gamma)^3} + \frac{3\gamma^{K-1} \Vmax}{(1-\gamma)^2}  \\
    \le& \frac{2C\apigreedyerror}{(1-\gamma)^2} + \frac{4C}{(1-\gamma)^3} \left( \sqrt{ \frac{419\Vmax^2\ln(|\Fcal||\Pcal|/\delta)}{3n} } + \sqrt{2\apicompleteerror} \right)+ \frac{6C\Vmax\tvmuerror}{(1-\gamma)^3} + \frac{3\gamma^{K-1} \Vmax}{(1-\gamma)^2}  \\
\end{align*}
The last step follows from that $a+b \le \sqrt{ 2(a^2+b^2)}$. The error bound is finished by simplifying the expression. The failure probability $3\delta$ is from the union bound of probability $\delta$ on which Assumption \ref{assume:muhat} fails, probability $\delta$ on which Lemma \ref{lem:concenration_2} fails, and the probability $\delta$ on which Lemma \ref{lem:concenration_policy_improvement} fails.
\end{proof}

Now we are going to use the fact that there is an almost no-value-loss projection from the \weaksettext to the \strongsettext in order to prove an error bound w.r.t any $\picomp \in \weakset$.
\setcounter{theorem}{0}
\begin{theorem}
Given an MDP $M = <\Scal, \Acal, R, P, \gamma, p>$, a dataset $D = \{ (s,a,r,s') \}$ with $n$ samples that is draw i.i.d. from $\mu \times R \times P$, and a finite Q-function classes $\Fcal$ and a finite policy class $\Pcal$ satisfying Assumption \ref{asm:api_completeness} and \ref{asm:weak_api_pi_realizalibity}, $\widehat{\pi}_t$ from Algorithm \ref{alg:api} satisfies that with probability at least $1-3\delta$,
\begin{align*}
    v^{\picomp}_{M} - v^{\widehat{\pi}_t}_{M} \le \frac{4C}{(1-\gamma)^3}\left(\sqrt{\frac{419\Vmax^2\ln\frac{|\Fcal||\Pcal| }{\delta}}{3n}} + 2\sqrt{\apicompleteerror}\right) + \frac{6C\Vmax\tvmuerror}{(1-\gamma)^3} + \frac{2C\apigreedyerror + 3\gamma^{K-1} \Vmax}{(1-\gamma)^2} + \frac{\Vmax\badstateprob}{1-\gamma}
\end{align*}
for any policy $\picomp \in \weakset$ and only take action over $\Acal$.
\end{theorem}
\setcounter{theorem}{3}

\begin{proof}
For any policy $\picomp$ that only take action over $\Acal$, Lemma \ref{lem:constrained_policy_same_value} tells that $v^{\picomp}_{M} \le v^{\projection(\picomp)}_{M'} + \frac{\Vmax\badstateprob}{1-\gamma}$. Since $\pi_t = \projection(\widehat{\pi}_t)$ and $\widehat{\pi}_t$ only takes action in $\Acal$, by Lemma \ref{lem:newmdp_same_value} and Lemma \ref{lem:projection_smaller_value} $v^{\widehat{\pi}_t}_{M} = v^{\widehat{\pi}_t}_{M'} \ge v^{\pi_t}_{M}$. Then $v^{\picomp}_{M} - v^{\widehat{\pi}_t}_{M} \le v^{\projection(\picomp)}_{M'} - v^{\pi_t}_{M'} + \frac{\Vmax\badstateprob}{1-\gamma}$ and Theorem \ref{thm:api_strongset} completes the proof.
\end{proof}

When there exist an optimal policy that is supported well by $\mu$. We can derive the following result about value gap between learned policy and optimal policy immediately from the main theorem about approximate policy iteration.

\begin{corollary}
If there exists an $\pi^\star$ on $M$ such that $\Pr(\mu(s,a) \le 2b|\piopt) \le \epsilon $.
then under the assumptions of Theorem \ref{thm:api}, $\widehat{\pi}_t$ from Algorithm \ref{alg:api} satisfies that with probability at least $1-3\delta$,
\begin{align*}
    v^{\piopt}_M - v^{\pi_t}_M \le & \frac{4C}{(1-\gamma)^3}\left(\sqrt{\frac{419\Vmax^2\ln\frac{|\Fcal||\Pcal| }{\delta}}{3n}} + 2\sqrt{\apicompleteerror}\right) + \frac{6C\Vmax\tvmuerror}{(1-\gamma)^3} \\
    & + \frac{2C\apigreedyerror + 3\gamma^{K-1} \Vmax}{(1-\gamma)^2} + \frac{\Vmax(\epsilon+C\tvmuerror)}{1-\gamma}
\end{align*} 
\end{corollary}
\begin{proof} Given the condition of $\pi^\star$,
\begin{align}
    \Pr\left(\muhat(s,a) \le b \Big| \piopt\right) \le & \Pr\left(\mu(s,a) \le 2b |\piopt\right) + \Pr\left( | \mu(s,a) - \muhat(s,a) | \ge b |\piopt\right)  \\
    \le & \epsilon + \Pr\left( | \mu(s,a) - \muhat(s,a) | \ge b |\piopt\right)\\
    \le & \epsilon + \frac{\EE_{\eta^{\piopt}} \left[ | \mu(s,a) - \muhat(s,a) | \right] }{ b } \\
    \le & \epsilon + \frac{ \densitybound d_\text{TV} (\mu(s,a), \muhat(s,a) ) }{ b } \\
    \le & \epsilon + C \tvmuerror
\end{align}
Then $\pi^\star \in \weakset$ with $\badstateprob = \epsilon +  C \tvmuerror$, and applying Theorem \ref{thm:api} finished the proof. 
\end{proof}

\subsection{Safe Policy Improvement Result}
\label{app:api_safe_policy_improvement}
In many scenarios we aim to have a policy improvement that is guaranteed to be no worse than the data collection policy, which is called safe policy improvement. By abusing the notation a bit, let $\mu(a|s)$ be a policy that generate the data set. For our algorithm, the safe policy improvement will hold if $\mu \in \weakset$. To show $\mu \in \weakset$, we only need that $\Pr(\mu(s,a)\le b | \mu) \le \badstateprob$. When the state-action space is finite, there must exist an minimum value for all non-zero $\mu(s,a)$'s. Let $\mu_{\min} = \min_{s,a s.t. \mu(s,a)>0} \mu(s,a) $. Then we have that, if $b \le \mu_{\min}$. $\Pr(\mu(s,a)\le b | \mu) = 0$. Thus we have:
\begin{corollary}
With finite state action space and $b \le \mu_{\min}$, under the assumptions as Theorem \ref{thm:api}, $\widehat{\pi}_t$ from Algorithm \ref{alg:api} satisfies that with probability at least $1-3\delta$,
\begin{align*}
    v^{\mu}_M - v^{\hat{\pi}_t}_M \le & \frac{ 52 \Vmax \sqrt{|\Scal||\Acal|} (\sqrt{\ln(2|\Scal||\Acal|/\delta)} +  \sqrt{\ln(1+n\Vmax)}) + 8 }{\sqrt{n}b(1-\gamma)^3} \\
    & + \frac{12\Vmax |\Scal| |\Acal|\ln(2|\Scal||\Acal|/\delta) }{nb(1-\gamma)^3} + \frac{3\gamma^{K-1} \Vmax}{(1-\gamma)^2}
\end{align*} 
\end{corollary}
\begin{proof}
In finite state action space, the number of all deterministic policies is less than $|\Acal|^{|\Scal|}$. Thus we have a policy class with $\apigreedyerror = 0$ and $|\Pcal|\le |\Acal|^{|\Scal|}$. Since the $Q$ value is bounded in $[0, \Vmax]$, we can construct a $\epsilon$ covering set $\Fcal$ of all value functions in $[0,\Vmax]^{|\Scal||\Acal|}$ with $(\frac{\Vmax}{\epsilon}+1)^{|\Scal||\Acal|}$ functions. Then $\avicompleteerror \le \max_g \min_{f \in \Fcal} \| f-g \|_{\mu,2} \le \max_g \min_{f \in \Fcal} \| f-g \|_{\infty} \le \epsilon $.

We can also bound $\tvmuerror$ in finite state action space. For any fixed $s,a$, by Berstein's inequality we have that with probability of 1-$\frac{\delta}{|\Scal||\Acal|}$:
\begin{align}
    |\muhat(s,a) - \mu(s,a)| &= \left| \frac{1}{n}\sum_{i=1}^n \mathds{1}(s^{(i)}=s,a^{(i)}=a) - \mathbb{E}[\mathds{1}(s^{(i)}=s,a^{(i)}=a)]\right| \\
    &\le \sqrt{\frac{2 \VV[\mathds{1}(s^{(i)}=s,a^{(i)}=a)]\ln(2|\Scal||\Acal|/\delta) }{n}} + \frac{4\ln(2|\Scal||\Acal|/\delta) }{n} \\
    &= \sqrt{\frac{2 \mu(s,a)(1-\mu(s,a)) \ln(2|\Scal||\Acal|/\delta) }{n}} + \frac{4\ln(2|\Scal||\Acal|/\delta) }{n} 
\end{align}
By taking summation of $|\muhat(s,a) - \mu(s,a)|$ and union bound over all $(s,a)$, we can bound the total variation bounds between $\muhat$ and $\mu$, with probability at least $1-\delta$,
\begin{align}
    \| \muhat - \mu \|_{TV} = & \frac{1}{2} \sum_{s,a} |\muhat(s,a) - \mu(s,a)| \\
    \le & \frac{1}{2} \sum_{s,a} \left( \sqrt{\frac{2 \mu(s,a)(1-\mu(s,a)) \ln(2|\Scal||\Acal|/\delta) }{n}} + \frac{4\ln(2|\Scal||\Acal|/\delta) }{n}  \right)  \\
    = & \frac{2|\Scal| |\Acal|\ln(2|\Scal||\Acal|/\delta) }{n} + \frac{1}{2} \sum_{s,a} \sqrt{\frac{2 \mu(s,a)(1-\mu(s,a)) \ln(2|\Scal||\Acal|/\delta) }{n}} \\
    \le & \frac{2|\Scal| |\Acal|\ln(2|\Scal||\Acal|/\delta) }{n} + \frac{1}{2}\sqrt{  \sum_{s,a} \frac{2 \mu(s,a) \ln(2|\Scal||\Acal|/\delta) }{n} \sum_{s,a} (1-\mu(s,a))  } \tag{Cauchy-Schwartz's inequality} \\
    = & \frac{2|\Scal| |\Acal|\ln(2|\Scal||\Acal|/\delta) }{n} + \frac{1}{2}\sqrt{  \frac{2 \ln(2|\Scal||\Acal|/\delta) }{n}  (|\Scal||\Acal|-1)  } \\
    \le & \frac{2|\Scal| |\Acal|\ln(2|\Scal||\Acal|/\delta) }{n} + \sqrt{  \frac{|\Scal||\Acal| \ln(2|\Scal||\Acal|/\delta) }{2n} } 
\end{align}

Now in a finite state action space we can construct the policy and $Q$ function sets with $|\Fcal| \le (\frac{\Vmax}{\epsilon}+1)^{|\Scal||\Acal|}$, $|\Pcal| \le |A|^{|S|}$, $\apigreedyerror = 0$, $\avicompleteerror \le \epsilon$, and bounded $\tvmuerror$. By plugging these terms into the result of Theorem \ref{thm:api}, we have the following bound:
\begin{align}
    v^{\mu}_M - v^{\hat{\pi}_t}_M \le & \frac{4C}{(1-\gamma)^3}\left(\sqrt{\frac{419\Vmax^2 (|\Scal|\ln|\Acal| + |\Scal||\Acal| \ln(1+\Vmax/\epsilon) + \ln(1/\delta) )  }{3n}}  + 2\sqrt{\epsilon}\right) \nonumber \\ & + \frac{6C\Vmax}{(1-\gamma)^3} \left( \sqrt{\frac{|\Scal||\Acal|\ln(2|\Scal||\Acal|/\delta) }{2n}} + \frac{2|\Scal| |\Acal|\ln(2|\Scal||\Acal|/\delta) }{n} \right) + \frac{3\gamma^{K-1} \Vmax}{(1-\gamma)^2}, 
\end{align}
for any chosen $\epsilon > 0$. So we can set that $\epsilon = 1/n$ to upper bound the the infimum of this upper bound.
\begin{align}
    v^{\mu}_M - v^{\hat{\pi}_t}_M 
    \le & \frac{4C}{(1-\gamma)^3}\left(\sqrt{\frac{419\Vmax^2 (|\Scal|\ln|\Acal| + |\Scal||\Acal| \ln(1+n\Vmax) + \ln(1/\delta) )  }{3n}}  + 2\sqrt{\frac{1}{n}}\right) \nonumber \\ & + \frac{6C\Vmax}{(1-\gamma)^3}\left( \sqrt{\frac{|\Scal||\Acal|\ln(2|\Scal||\Acal|/\delta) }{2n}} + \frac{2|\Scal| |\Acal|\ln(2|\Scal||\Acal|/\delta) }{n} \right) + \frac{3\gamma^{K-1} \Vmax}{(1-\gamma)^2}
\end{align}
Notice that in discrete space we have that $\densitybound \le 1$. By replacing $C$ with $1/b$ and simplify some terms, we have that:
\begin{align*}
    v^{\mu}_M - v^{\hat{\pi}_t}_M \le & \sqrt{\frac{6704\Vmax^2|\Scal| (\ln(|\Acal|/\delta) + |\Acal| \ln(1+n\Vmax))  }{3nb^2(1-\gamma)^6}}  + \frac{8}{b\sqrt{n}(1-\gamma)^3}  \nonumber \\ & + \sqrt{\frac{18\Vmax^2|\Scal||\Acal|\ln(2|\Scal||\Acal|/\delta) }{nb^2(1-\gamma)^6}} + \frac{12\Vmax |\Scal| |\Acal|\ln(2|\Scal||\Acal|/\delta) }{nb(1-\gamma)^3} + \frac{3\gamma^{K-1} \Vmax}{(1-\gamma)^2} \\
    \le & \frac{ 52 \Vmax \sqrt{|\Scal||\Acal|} (\sqrt{\ln(2|\Scal||\Acal|/\delta)} +  \sqrt{\ln(1+n\Vmax)}) + 8 }{\sqrt{n}b(1-\gamma)^3} \\
    & + \frac{12\Vmax |\Scal| |\Acal|\ln(2|\Scal||\Acal|/\delta) }{nb(1-\gamma)^3} + \frac{3\gamma^{K-1} \Vmax}{(1-\gamma)^2}
\end{align*} 
\end{proof}

\section{Proofs for $Q$ Iteration Guarantees}
\label{appendix:fqi}
In this section, we are going to prove the result of Theorem \ref{thm:fqi}.  
We will first give a proof sketch before we start the proof formally.
The proof follows a similar structural as the policy iteration case. To prove Theorem \ref{thm:fqi} we first prove a similar version of Theorem \ref{thm:fqi} but the comparator polices are in \strongsettext (formally stated as Theorem \ref{thm:fqi_strongset} later). Then we show an upper bound of $v^{\pi}_{M'} - v^{\pi_t}_{M'}$ where $\pi \in \strongset$ and $\pi_t$ is the output of algorithm (Theorem \ref{thm:fqi_strongset}, will be formally stated later). Then we are going to show that for any policy $\pi$ in the \weaksettext, after a projection $\projection$ it is in the \strongsettext and $v^{\pi}_M \le v^{\projection(\pi)}_{M'} + \Vmax\badstateprob/(1-\gamma)$. Then we can provide the upper bound for $v^{\pi}_{M} - v^{\pi_t}_{M}$ for any $\pi$ in \weaksettext (Theorem \ref{thm:fqi}).

The proof sketch of Theorem \ref{thm:fqi_strongset} goes as follow. One key step to prove this error bound is to convert the performance difference between any policy $\picomp \in \strongset$ and $\pi_t$ to a value function gap that is filtered by $\zeta$:
$$
    v^{\picomp} - v^{\pi_t}\leq  \| \zeta \left(Q^{\picomp}- f_t \right)\|_{1, \nu_1}/(1-\gamma),
$$
where $\nu_1$ is some admissible distribution over $\Scal \times \Acal$. The filter $\zeta$ allows the change of measure from $\nu_1$ to $\mu$ without constraining the density ratio between an arbitrary distribution $\nu$ and $\mu$. Instead for any $s,a$ where $\zeta$ is one, by definition $\mu$ is lower bounded and the density ratio is bounded by $C$ (details in Lemma \ref{lem:decompose}).

The rest of the proof has a similar structure with the standard FQI analysis. In Lemma \ref{lem:iteration}, we bound the norm $\| \zeta (Q^{\picomp}- f_t )\|_{2, \nu_1}$ by $C\| ( f_t - \OurT f_t )\|_{2, \mu}/(1-\gamma)$ and one additional sub-optimality error $\| Q^{\picomp} - \OurT Q^{\picomp} \|_{2,\mu}$. The additional sub-optimality error term comes from the fact that $\picomp$ may not be an optimal policy since the optimal policy may not be a $\zeta$-constrained policy. 
The last step to finish the proof is to bound the expected Bellman residual by concentration inequality. Lemma \ref{lem:concenration_1} shows how to bound that following a similar approach as \cite{chen2019information}. Then the main theorem is proved by combine all those steps. After that we prove when we can bound the value gap with resepct to optimal value in Corollary \ref{cor:fqi_optgap}.

Now we start the proof. We are going to condition on the high probability bounds in Assumption \ref{assume:muhat} holds when we proof the lemmas.
\begin{lemma}
\label{lem:decompose}
	For $\pi_t = \projection(\widehat{\pi}_t)$ in Algorithm \ref{alg:fqi}, for any policy $\picomp \in \strongset$ we have
	$$v^{\picomp} - v^{\pi_t}\leq \sum_{h=0}^\infty \gamma^h \left(\left\| \zeta \left(Q^{\picomp}- f_t \right)\right\|_{1, \eta^{\pi_t}_h \times \picomp} + \left\| \zeta\left(Q^{\picomp} - f_t \right) \right\|_{1, \eta^{\pi}_h \times \pi_t}\right).$$
\end{lemma}
\begin{proof}
Given a deterministic greedy policy $\widehat{\pi}_t$, $\pi_t = \projection(\widehat{\pi}_t)$ is also a deterministic policy and $\pi_t(s)$ equals $\widehat{\pi}_t(s)$ unless $\zeta(s,\widehat{\pi}_t(s)) = 0$, where $\pi_t(s)=a_\text{abs}$. Notice $\widehat{\pi}_t(s)$ is the maximizer of $\zeta(s,\cdot)f_t(s,\cdot)$. If $\zeta(s,\widehat{\pi}_t(s)) = 0$ then $\zeta(s,a)f_t(s,a) = 0$ for all $a$. We have that $\pi_t(s)$ is also the maximizer of $\zeta(s,\cdot)f_t(s,\cdot)$. 
	\begin{align}
	v^{\picomp} - v^{\pi_t} 
	= & \sum_{h=0}^\infty \gamma^h \EE_{s\sim \eta^{\pi_t}_h} [Q^{\picomp}(s,\picomp) - Q^{\picomp}(s,\pi_t)] \tag{\cite[Lemma 6.1]{kakade2002approximately}} \\
	\le & \sum_{h=0}^\infty \gamma^h \EE_{s\sim \eta^{\pi_t}_h} \left[ \zeta(s,\picomp) Q^{\picomp}(s,\picomp) - \zeta(s,\pi_t) Q^{\picomp}(s,\pi_t) \right] \label{eq:avi_add_zeta} \\
	\le & \sum_{h=0}^\infty \gamma^h \EE_{s\sim \eta^{\pi_t}_h} [\zeta(s,\picomp)Q^{\picomp}(s, \picomp) - \zeta(s,\picomp)f_t(s, \picomp) + \zeta(s,\pi_t)f_t(s, \pi_t) - \zeta(s,\pi_t)Q^{\picomp}(s,\pi_t)] \label{eq:avi_greedy_improvement} \\
	\le & \sum_{h=0}^\infty \gamma^h \left(\left\| \zeta \left(Q^{\picomp}- f_t \right)\right\|_{1, \eta^{\pi_t}_h \times \picomp} + \left\| \zeta\left(Q^{\picomp} - f_t \right) \right\|_{1, \eta^{\pi_t}_h \times \pi_t}\right) 
	\end{align}
Equation (\ref{eq:avi_add_zeta}) follows from the fact that for all $s,a$ such that $\picomp (a|s)>0$, either $\zeta(s,a) = 1$, or $a=a_{\text{abs}}$. $a=a_{\text{abs}}$ indicates that $Q^{\picomp}(s,a) = 0$. So for all $s,a$ such that $\picomp (a|s)>0$, $Q^{\picomp}(s,a) = \zeta(s,a)Q^{\picomp}(s,a)$. The second part follows from that for any $s,a$, $Q^{\picomp}(s,a) \ge \zeta(s,a)Q^{\picomp}(s,a)$. Equation (\ref{eq:avi_greedy_improvement}) follows from the fact that $\pi_t(s)$ is the maximizer of $\zeta(s,\cdot)f_t(s,\cdot)$.
\end{proof}

\begin{lemma} 
\label{lem:piff}
	For any two function $f_1, f_2: \Sext \times \Aext \to \RR^{+}$, define $\pi_{f_1,f_2}(s) = \argmax_{a \in \Acal} \left| f_1(s,a) - f_2(s,a) \right|$. Then we have $\forall \nu: \Sext \to \Delta(\Aext)$,
		$$
		\left\|\max_{a\in \Acal} f_1 - \max_{a\in \Acal} f_2 \right\|_{1,P(\nu)} \le \|f_1 - f_2 \|_{1,P(\nu) \times \pi_{f_1,f_2}}. 
		$$
\end{lemma}

\begin{proof}
\begin{align*}
\left\|\max_{a\in \Acal} f_1 - \max_{a\in \Acal} f_2 \right\|_{1,P(\nu)}
&= \EE_{s \sim P(\nu)} \left| \max_{a\in \Acal} f_1(s,a) - \max_{a\in \Acal} f_2(s,a) \right| \\
&\le \EE_{s \sim P(\nu)} \max_{a \in \Acal} \left| f_1(s,a) -  f_2(s,a) \right| \\
&=  \EE_{s \sim P(\nu), a \sim \pi_{f_1,f_2}}\left| f_1(s,a) -  f_2(s,a) \right| \\
&= \|f_1 - f_2\|_{1,P(\nu) \times \pi_{f_1,f_2}}^2. 
\end{align*}

\end{proof}

\begin{lemma}\label{lem:iteration}
For the data distribution $\mu$ and any admissible distribution $\nu$ over $\Sext \times \Aext$, $f, f':\Scal\times\Acal\rightarrow\mathbb{R^+}$ and any $\picomp \in \strongset$, we have 
\begin{align*}
    \left\| \zeta \left( f - Q^{\picomp} \right) \right\|_{1,\nu} \leq ~& C \left( \left\| f - \OurT f' \right\|_{2,\mu} + \left\| \OurT Q^{\picomp}  -  Q^{\picomp} \right\|_{2,\mu} + \Vmax{\tvmuerror} \right) \\
    ~& + \gamma \left\| \zeta \left(f' -  Q^{\picomp}\right)\right\|_{2,P(\nu) \times \pi_{\zeta f', \zeta Q^{\picomp}}}.
    \end{align*}
\end{lemma}

\begin{proof}
\begin{align}
	& \left\| \zeta \left( f - Q^{\picomp} \right) \right\|_{1,\nu} \\
	= & \left\| \zeta \left( f - \OurT f' +   \OurT f' - \OurT Q^{\picomp} + \OurT Q^{\picomp}  -  Q^{\picomp} \right) \right\|_{1,\nu} \\
	\le & \left\| \zeta \left(f - \OurT f' \right) \right\|_{1,\nu} + \left\| \zeta \left(\OurT f' - \OurT Q^{\picomp} \right) \right\|_{1,\nu} + \left\| \zeta \left( \OurT Q^{\picomp}  -  Q^{\picomp} \right) \right\|_{1,\nu} \\
	\le & C \left\| f - \OurT f' \right\|_{1,\muhat} + \gamma \left\| \max_{a \in \Acal} \zeta f' - \max_{a \in \Acal} \zeta Q^{\picomp}  \right\|_{1,P(\nu)} + C \left\| \OurT Q^{\picomp}  -  Q^{\picomp} \right\|_{1,\muhat} \\
	\le & 2C\Vmax\tvmuerror + C \left\| f - \OurT f' \right\|_{1,\mu} + \gamma \left\| \max_{a \in \Acal} \zeta f' - \max_{a \in \Acal} \zeta Q^{\picomp}  \right\|_{1,P(\nu)} + C \left\| \OurT Q^{\picomp}  -  Q^{\picomp} \right\|_{1,\mu} \\
	\le & C \left( \left\| f - \OurT f' \right\|_{2,\mu} + \left\| \OurT Q^{\picomp}  -  Q^{\picomp} \right\|_{1,\mu} + 2\Vmax\tvmuerror \right) + \gamma \left\| \zeta \left(f' -  Q^{\picomp}\right)\right\|_{1,P(\nu) \times \pi_{\zeta f', \zeta Q^{\picomp}}}
\end{align}
The change of norms from $\| \cdot \|_{\nu}$ to $\| \cdot \|_{\mu}$ follows from that $\zeta(s,a) \neq 0$ iff $\muhat(s,a) \ge b$ and thus $\nu(s,a) \le \muhat(s,a)\densitybound/b = C\muhat(s,a)$. The last step follows from Lemma \ref{lem:piff}. $\left\| \zeta \left(\OurT f' - \OurT Q^{\picomp} \right) \right\|_{1,\nu} \le \gamma \left\| \max_{a \in \Acal} \zeta f' - \max_{a \in \Acal} \zeta Q^{\picomp}  \right\|_{1,P(\nu)}$ follows from:
\begin{align}
    \left\| \zeta \left(\OurT f' - \OurT Q^{\picomp} \right) \right\|_{1,\nu}
    = & \EE_{(s,a) \sim \nu} \left[\zeta(s,a)\left|\OurT f'(s,a) - \OurT Q^{\picomp} (s,a)\right|\right] \\
	\le & \EE_{(s,a) \sim \nu} \left[\left|\OurT f'(s,a) - \OurT Q^{\picomp} (s,a)\right| \right] \\
    = & \EE_{(s,a) \sim \nu} \left[\left| \gamma \mathbb{E}_{s' \sim P(s,a)} \max_{a'\in \Acal} \zeta(s',a') f'(s',a') - \max_{a'\in \Acal} \zeta(s',a') Q^{\picomp} (s',a') \right|\right] \\
	\le & \gamma \, \EE_{(s,a) \sim \nu, s' \sim P(s,a)}\left[\left|\max_{a'\in \Acal} \zeta(s',a')f'(s',a') - \max_{a'\in \Acal} \zeta(s',a')Q^{\picomp} (s',a') \right|\right] \tag{Jensen} \\
	= &~ \gamma \, \EE_{s' \sim P(\nu)}\left[\left|\max_{a'\in \Acal} \zeta(s',a')f'(s',a') - \max_{a'\in \Acal} \zeta(s',a')Q^{\picomp} (s',a') \right|\right]\\
	= & \gamma \left\| \max_{a \in \Acal} \zeta f' - \max_{a \in \Acal} \zeta Q^{\picomp}  \right\|_{1,P(\nu)}
\end{align}
\end{proof}

Now we are going to use Berstein's inequality to bound $\left\| f_{t+1} - \OurT f_t \right\|_{2,\mu}$, which mostly follows from \cite{chen2019information}'s proof for the vanilla value iteration. 

\begin{lemma}
\label{lem:concenration_1} 
With Assumption \ref{asm:avi_completeness} holds, let $g_f^\star=\argmin_{g\in\Fcal}\|g-\OurT f\|_{2,\mu}$, then $\|g_f^\star-\OurT f\|_{2,\mu}^2\le \avicompleteerror$. The dataset $D$ is generated i.i.d. from $M$ as follows: $(s,a) \sim \muu$, $r = R(s,a)$, $s' \sim P(s,a)$. Define $\Lcal_\muu( f ;f') = \mathbb{E} [\Lcal_D( f ;f') ]$.
We have that $\forall f\in\Fcal$, with probability at least $1-\delta$,
\begin{align*}
\Lcal_\muu( \OurT_D f;f)-\Lcal_\muu(g_f^\star;f) \le \frac{208\Vmax^2\ln\frac{|\Fcal| }{\delta}}{3n}+ \avicompleteerror
\end{align*}
where $\OurT_D f = \argmin_{g \in \Fcal} \Lcal_D(g, f)$.
\end{lemma}

\begin{proof}
    This proof is similar with the proof of Lemma \ref{lem:concenration_2}, and we adapt it to operator $\OurT$. The only change is the definition of $V_f(\cdot)$ and $X(\cdot,\cdot,\cdot)$. The definition of $\Lcal_D$ and $\Lcal_{\mu}$ would not change between $M$ and $M'$, and the right hand side is also the same constant for $M$ and $M'$. So the result we prove here does not change from $M$ to $M'$.

    For the simplicity of notations, let $V_{f}(s) = \max_{a \in \Acal} \zeta(s,a) f(s,a)$. Fix $f, g\in\Fcal$, and define $$ X(g,f,g_f^\star) := \left(g(s,a) - r - \gamma V_{f}(s')\right)^2 - \left(g_f^\star(s,a) - r - \gamma V_{f}(s')\right)^2.$$
	Plugging each $(s,a,r,s') \in D$ into $X(g,f,g_f^\star)$, we get i.i.d.~variables $X_1(g,f,g_f^\star), X_2(g,f,g_f^\star), \ldots,$ $X_n(g,f,g_f^\star)$. It is easy to see that 
	$$
	\frac{1}{n}\sum_{i=1}^n X_i(g,f,g_f^\star) = \Lcal_{D}(g;f) - \Lcal_{D}(g_f^\star;f).
	$$
	By the definition of $\Lcal_\mu$, it is also easy to see that
	$$
    \Lcal_{\mu}(g; f) = \left\| g - \OurT f \right\|_{2,\mu}^2 + \mathbb{E}_{s,a \sim \mu} \left[ \mathbb{V}_{r,s'} \left(r + \gamma \max_{a' \in \Acal} \zeta(s',a')f(s',a') \right) \right]
    $$
	Notice that the second part does not depends on $g$. Then
	$$
	\Lcal_{\muu}(g; f) - \Lcal_{\muu}(\OurT f; f) = \left\|g-\OurT f\right\|_{2,\muu}^2
	$$
	Then we bound the variance of $X$:
	\begin{align}
	\VV[X(g,f,g_f^\star)] \le &~ \EE[X(g,f,g_f^\star)^2] \notag\\
	= &~ \EE_{\mu} \left[\left(\big(g(s,a) - r - \gamma V_{f}(s')\big)^2 - \big(g_f^\star(s,a) - r - \gamma  V_{f}(s')\big)^2\right)^2\right] \notag\\
	= &~ \EE_{\mu} \left[\big(g(s,a)  - g_f^\star(s,a) \big)^2 \big(g(s,a) + g_f^\star(s,a) - 2r - 2\gamma V_{f}(s') \big)^2\right] \notag\\
	\le &~ 4\Vmax^2~ \EE_{\mu} \left[\big(g(s,a)  - g_f^\star(s,a) \big)^2\right] \notag\\
	= &~ 4\Vmax^2~ \|g - g_f^\star\|_{2,\muu}^2 \label{eq:bound_var}\\
	\le &~ 8\Vmax^2~ (\EE[X(g,f,g_f^\star)]+2\avicompleteerror) \tag{*}.
	\end{align}
	Step (*) holds because 
	\begin{align*}
	&~\|g-g_f^\star\|_{2,\muu}^2\\
	\leq&~2\left(\|g- \OurT f\|_{2,\muu}^2+\|\OurT f-g_f^\star\|_{2,\muu}^2\right) \tag{$(a+b)^2\leq 2a^2+2b^2$} \\
	\leq&~2\left(\|g-\OurT f\|_{2,\muu}^2-\|\OurT f-g_f^\star\|_{2,\muu}^2+2\|\OurT f-g_f^\star\|_{2,\muu}^2\right)\\
	=&~2\left[(\Lcal_{\muu}(g; f) - \Lcal_{\muu}(\OurT f; f))-(\Lcal_{\muu}(g_f^\star; f) - \Lcal_{\muu}(\OurT f; f))+2\|\OurT f-g_f^\star\|_{2,\muu}^2\right]\\
	=&~2\left[(\Lcal_{\muu}(g; f) - \Lcal_{\muu}(g_f^\star; f) +2\|\OurT f-g_f^\star\|_{2,\muu}^2\right]\\
	=&~2\left(\EE[X(g,f,g_f^\star)]+2\|\OurT f-g_f^\star\|_{2,\muu}^2\right)\\
	\leq &~2(\EE\left[X(g,f,g_f^\star)\right]+2\avicompleteerror)
	\end{align*}
	
	Next, we apply (one-sided) Bernstein's inequality and union bound over all $f\in\Fcal$ and $g\in\Fcal$. With probability at least $1-\delta$, we have
	\begin{align*}
	\EE[X(g, f,g_f^\star)] - \frac{1}{n}\sum_{i=1}^n X_i(g,f,g_f^\star) \le &~ \sqrt{\frac{2 \VV[X(g,f,g_f^\star)] \ln\tfrac{|\Fcal|^2}{\delta}}{n}} + \frac{4\Vmax^2 \ln\tfrac{|\Fcal|^2}{\delta}}{3n} \\
	= &~ \sqrt{\frac{32 \Vmax^2 \left(\EE[X(g, f,g_f^\star)]+2\avicompleteerror\right) \ln\tfrac{|\Fcal|}{\delta}}{n}} + \frac{8\Vmax^2 \ln\tfrac{|\Fcal|}{\delta}}{3n} 
	\end{align*}
	Since $\OurT_D f$ minimizes $\Lcal_D(\holder; f)$, it also minimizes $\frac{1}{n}\sum_{i=1}^n X_i(\cdot,f,g_f^\star)$. This is because the two objectives only differ by a constant $\Lcal_D(g_f^\star; f)$. Hence,
	$$
	\frac{1}{n}\sum_{i=1}^n X_i( \OurT_D f, f,g_f^\star) \le 
	\frac{1}{n}\sum_{i=1}^n X_i(g_f^\star, f,g_f^\star) = 0.
	$$
	Then,
	\begin{align*}
	\EE[X(\OurT_D f,f,g_f^\star)]\le \sqrt{\frac{32 \Vmax^2 \left(\EE[X(\OurT_D f, f,g_f^\star)]+2\avicompleteerror\right) \ln\tfrac{|\Fcal|}{\delta}}{n}} + \frac{8\Vmax^2 \ln\tfrac{|\Fcal|}{\delta}}{3n}.
	\end{align*}
	Solving for the quadratic formula,
	\begin{align*}
	\EE[X(\OurT_D f, f,g_f^\star)]\le&~ \sqrt{48\left(\frac{8\Vmax^2\ln\frac{|\Fcal|}{\delta}}{3n}\right)^2+
		\frac{64\Vmax^2\ln\frac{|\Fcal| }{\delta}}{n} \avicompleteerror}+\frac{56\Vmax^2\ln\frac{|\Fcal|}{\delta}}{3n}\\
	\leq&~\frac{(56+32\sqrt{3})\Vmax^2\ln\frac{|\Fcal| }{\delta}}{3n}+\sqrt{\frac{64\Vmax^2\ln\frac{|\Fcal| }{\delta}}{n}\avicompleteerror} \tag{$\sqrt{a+b}\leq\sqrt{a}+\sqrt{b}$ and $\ln \frac{|\Fcal|}{\delta}>0$}\\
	\leq&~\frac{112\Vmax^2\ln\frac{|\Fcal| }{\delta}}{3n}+\sqrt{\frac{64\Vmax^2\ln\frac{|\Fcal| }{\delta}}{n}\avicompleteerror} \\
	\leq&~\frac{112\Vmax^2\ln\frac{|\Fcal| }{\delta}}{3n}+\frac{32\Vmax^2\ln\frac{|\Fcal| }{\delta}}{n} + \avicompleteerror \\
	\leq&~\frac{208\Vmax^2\ln\frac{|\Fcal| }{\delta}}{3n}+ \avicompleteerror
	\end{align*}
	Noticing that $\EE [X(\OurT_D f; f,g_f^\star)]=\Lcal_\muu(\OurT_D f; f)-\Lcal_\muu(g_f^\star;f)$, we complete the proof.
\end{proof}

Now we could prove the main theorem about fitted Q iteration. 

\begin{theorem}
\label{thm:fqi_strongset}
Given a MDP $M = <\Scal, \Acal, R, P, \gamma, p>$, a dataset $D = \{ (s,a,r,s') \}$ with $n$ samples that is draw i.i.d. from $\mu \times R \times P$, and a finite Q-function classes $\Fcal$ satisfying Assumption \ref{asm:avi_completeness}, $\pi_{t} = \projection(\widehat{\pi}_t)$ from Algorithm \ref{alg:fqi} satisfies that with probability at least $1-2\delta$, $v^{\picomp} - v^{\pi_t} \le$
\begin{align*}
     \frac{2C}{(1-\gamma)^2}\left( \sqrt{\frac{208\Vmax^2\ln\frac{|\Fcal|}{\delta}}{3n}} +2\sqrt{\avicompleteerror} + \Vmax\tvmuerror + \left\| Q^{\picomp} - \OurT Q^{\picomp} \right\|_{1,\mu}\right) + \frac{2\gamma^t  \Vmax}{1-\gamma}
\end{align*}
for any policy $\picomp \in \strongset$.
\end{theorem}

\begin{proof}
Firstly, we can let $f=f_t$ and $f'=f_{t-1}$ in Lemma \ref{lem:iteration}. This gives us that 
$$\left\|f_t - Q^{\picomp}\right\|_{1,\nu}\leq C \left( \|f_t - \OurT f_{t-1}\|_{2,\muu} + \left\| Q^{\picomp} - \OurT Q^{\picomp} \right\|_{1,\mu} + 2\Vmax\tvmuerror\right) + \gamma \| f_{t-1} - Q^{\picomp}\|_{1,P(\nu) \times \pi_{f_{k-1}, Q^{\picomp}}}$$

Note that we can apply the same analysis on $P(\nu) \times \pi_{\fo, Q^\star}$ and expand the inequality $t$ times. It then suffices to upper bound $\|f_{t} - \OurT f_{t-1}\|_{2,\muu}$.
\begin{align*}
	&~\|f_t - \OurT f_{t-1}\|_{2,\muu}^2\\
	= &~ \Lcal_{\muu}(f_t; f_{t-1}) - \Lcal_{\muu}(\OurT f_{t-1}; f_{t-1}) \tag{Definition of $\Lcal_\mu$}\\ 
	= &~ [\Lcal_{\muu}(f_t; f_{t-1}) -\Lcal_{\muu}(g_{f_{t-1}}^\star; f_{t-1})]+[\Lcal_{\muu}(g_{f_{t-1}}^\star; f_{t-1})- \Lcal_{\muu}(\OurT f_{t-1}; f_{t-1})] \\ 
	\le &~ \epsilon_4 + \|g_{f_{t-1}}^\star-\OurT f_{t-1}\|_{2,\muu}^2 \tag{Lemma \ref{lem:concenration_1} and definition of $\Lcal_\mu$}  \\
	\le &~ \epsilon_4 +\avicompleteerror. \tag{Definition of $g_{Q_{k-1}}^\star$ and Assumption \ref{asm:avi_completeness}}
\end{align*}
The inequality holds with probability at least $1-\delta$ and $\epsilon_4=\frac{208\Vmax^2\ln\frac{|\Fcal|}{\delta}}{3n}+\avicompleteerror$. Noticing that $\epsilon_4$ and $\avicompleteerror$ do not depend on $t$, and the inequality holds simultaneously for different $t$, we have that 
$$
\|f_t - Q^{\picomp} \|_{1,\nu} \le  \frac{1-\gamma^t}{1-\gamma} C\left( \sqrt{(\epsilon_4+\avicompleteerror)} + \Vmax\tvmuerror + \left\| Q^{\picomp} - \OurT Q^{\picomp} \right\|_{1,\mu}\right) + \gamma^t \Vmax.
$$
Applying this to Lemma \ref{lem:decompose}, we have that
\begin{align*}
&v^{\picomp} - v^{\pi_{t}} \\
\le&  \frac{2}{1-\gamma}\left(\frac{1-\gamma^t}{1-\gamma} C\left( \sqrt{(\epsilon_4+\avicompleteerror)} + \Vmax\tvmuerror + \left\| Q^{\picomp} - \OurT Q^{\picomp} \right\|_{1,\mu}\right) + \gamma^t \Vmax\right)\\
\le& \frac{2C}{(1-\gamma)^2}\left( \sqrt{\epsilon_4+\avicompleteerror} + \Vmax\tvmuerror + \left\| Q^{\picomp} - \OurT Q^{\picomp} \right\|_{1,\mu}\right) + \frac{2\gamma^t  \Vmax}{1-\gamma}\\
\le& \frac{2C}{(1-\gamma)^2}\left( \sqrt{\frac{208\Vmax^2\ln\frac{|\Fcal|}{\delta}}{3n}} +2\sqrt{\avicompleteerror} + \Vmax\tvmuerror + \left\| Q^{\picomp} - \OurT Q^{\picomp} \right\|_{1,\mu}\right) + \frac{2\gamma^t  \Vmax}{1-\gamma}.
\end{align*}
\end{proof}

Now we are going to use the fact that there is an no-value-loss projection from the \weaksettext to the \strongsettext to prove an error bound w.r.t any $\picomp \in \weakset$.

\setcounter{theorem}{1}
\begin{theorem}
Given a MDP $M = <\Scal, \Acal, R, P, \gamma, p>$, a dataset $D = \{ (s,a,r,s') \}$ with $n$ samples that is draw i.i.d. from $\mu \times R \times P$, and a finite Q-function classes $\Fcal$ satisfying Assumption \ref{asm:avi_completeness}, $\widehat{\pi}_t$ from Algorithm \ref{alg:fqi} satisfies that with probability at least $1-2\delta$, $v^{\picomp} - v^{\widehat{\pi}_t} \le$
\begin{align*}
     \frac{2C}{(1-\gamma)^2}\left( \sqrt{\frac{208\Vmax^2\ln\frac{|\Fcal|}{\delta}}{3n}} +2\sqrt{\avicompleteerror} + \Vmax\tvmuerror + \left\| Q^{\picomp} - \OurT Q^{\picomp} \right\|_{2,\mu}\right) + \frac{(2\gamma^t + \badstateprob) \Vmax}{1-\gamma}
\end{align*}
for any policy $\picomp \in \weakset$.
\end{theorem}

\begin{proof}
The difference between this theorem and Theorem \ref{thm:fqi_strongset} is that $\picomp$ is in $\weakset$ which is significantly larger than $\strongset$.

This prove mimics the proof of Theorem \ref{thm:api}. For any policy $\picomp \in \weakset$, Lemma \ref{lem:constrained_policy_same_value} tells that $v^{\picomp}_{M} \le v^{\projection(\picomp)}_{M'} + \frac{\Vmax\badstateprob}{1-\gamma}$. Since $\pi_t = \projection(\widehat{\pi}_t)$, $v^{\widehat{\pi}_t}_{M} = v^{\widehat{\pi}_t}_{M'} \ge v^{\pi_t}_{M}$. Then $v^{\picomp}_{M} - v^{\widehat{\pi}_t}_{M} \le v^{\projection(\picomp)}_{M'} - v^{\pi_t}_{M'} + \frac{\Vmax\badstateprob}{1-\gamma}$ and Theorem \ref{thm:fqi_strongset} completes the proof.
\end{proof}

\noindent \textbf{Remark:} The first term in the theorem comes from that the best policy in the \weaksettext is not optimal. Note that the \weaksettext does not requires any realizability to do with our function approximation but merely about the density ratio of a policy. When there is an optimal policy of $M$ such in $\weakset$, we have the same type of bound as standard approximate value iteration analysis.

\begin{corollary}
If there exists an $\pi^\star$ on $M$ such that $\Pr(\mu(s,a) \le 2b|\piopt) \le \epsilon $.
then under the condition as Theorem \ref{thm:fqi}, $\widehat{\pi}_t$ from Algorithm \ref{alg:fqi} satisfies that with probability at least $1-2\delta$, $ v^{\piopt}_{M} - v^{\pi_t}_{M} \le$
\begin{align*}
    \frac{2C}{(1-\gamma)^2}\left( \sqrt{\frac{208\Vmax^2\ln\frac{|\Fcal|}{\delta}}{3n}} +2\sqrt{\avicompleteerror} + \Vmax\tvmuerror + \left\| Q^{\piopt} - \OurT Q^{\piopt} \right\|_{2,\mu}\right) + \frac{\Vmax(2\gamma^t+\epsilon+CU\tvmuerror)}{1-\gamma} 
\end{align*}
\end{corollary}
\begin{proof}
The proof of $\piopt \in \weakset$ is same as the proof in Corollary \ref{cor:api_optgap}. Then proof is finished by applying Theorem \ref{thm:fqi}.
\end{proof}

\section{Details of CartPole Experiment}
\label{appendix:cartpole}
\subsection{Full results of Discretized CartPole-v0}
\begin{figure}[ht]
    \centering
    \includegraphics[width=\textwidth]{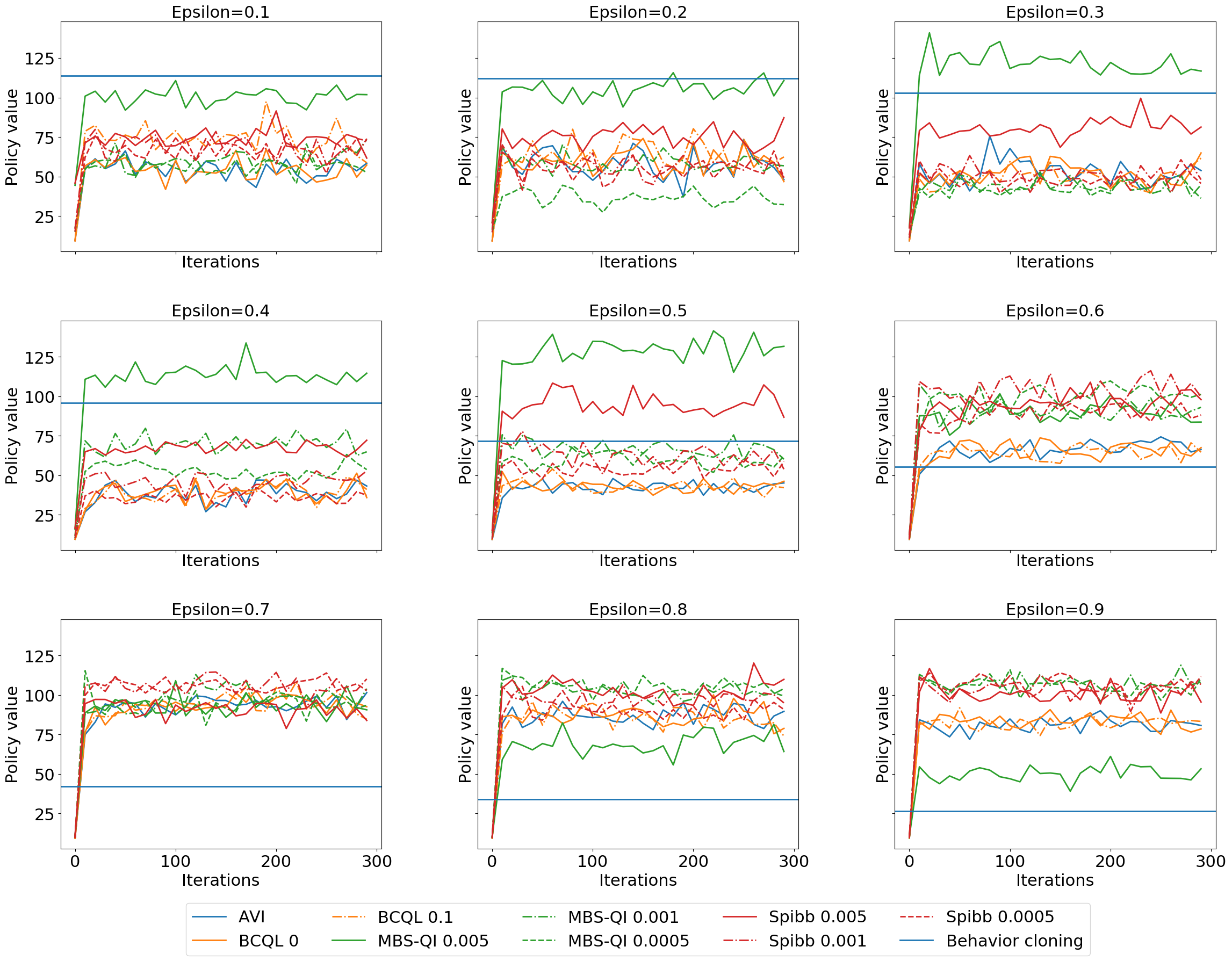}
    \caption{CartPole-v0 with discretized state space. The learning curve of all algorithms with different hyper-parameters, data generated with different $\epsilon$-greedy behavior policy. The hyper-parameter of SPIBB \cite{laroche2019safe} and MBS-QI is the threshold of $\muhat(s,a)$ and the hyper-parameter of BCQL \cite{fujimoto2019off} is the threshold of $\muhat(a|s)$. }
    \label{fig:cp_all_eps}
\end{figure}

In section \ref{sec:discrete_cartpole}, we compare AVI, BCQL\cite{fujimoto2019off}, SPIBB\cite{laroche2019safe}, Behavior cloning and our algorithm MBS-QI, in CartPole-v0 with discretized state space. The data is generated by a $\epsilon$-greedy policy ($\epsilon$ from $0.1$ to $0.9$) and we report the resulting policies from different algorithm with the best hyper-parameter in each $\epsilon$. In this section we show the learning curve for each $\epsilon$ and each hyper-parameter value. We run the BCQ algorithm with the threshold of $\muhat(a|s)$ in $\{ 0, 0.1\}$, and we run the SPIBB algorithm and MBS-QI with the threshold of $\muhat(s,a)$ in $\{ 0.005, 0.001, 0.0005\}$. Figure \ref{fig:cp_all_eps} shows for most of the $\epsilon$ and threshold our algorithm tie with the best baseline (SPIBB), and the best threshold of our algorithm outperform all baseline algorithms.

In Figure \ref{fig:cp_all_eps}, we observe the trend that smaller $\epsilon$ will prefer a smaller $b$. This is verified by more results in the next section, and we discuss the reasons for this phenomenon there.

\subsection{Ablation study of threshold b}
A key aspect of our algorithm is to filter the state space by a threshold on the estimated probability $\muhat(s,a)$. This prevents the algorithm from updating using low-confidence state, action pairs when bootstrapping values. Then the choice of threshold $b$ is a key trade-off in our algorithm: if $b$ is too small it can not remove the low-confident state, action pairs effectively; if $b$ is too large it might remove too many state, action pairs and prevent learning from more data. In order to demonstrate the effect of $b$ and how should we choose b in different settings, we show the performance of MBS-QI in a larger range of $b$ and several $\epsilon$ values.

\begin{figure*}[ht]
    \begin{minipage}{0.32\textwidth}
    \centering
    \includegraphics[width=\textwidth]{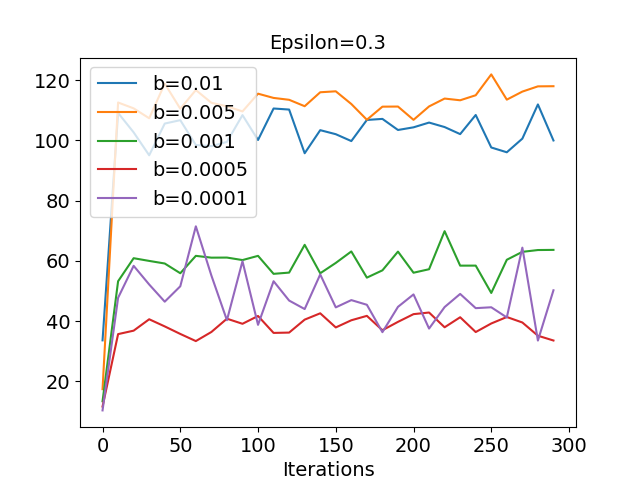}
    \end{minipage}
    \begin{minipage}{0.32\textwidth}
    \centering
    \includegraphics[width=\textwidth]{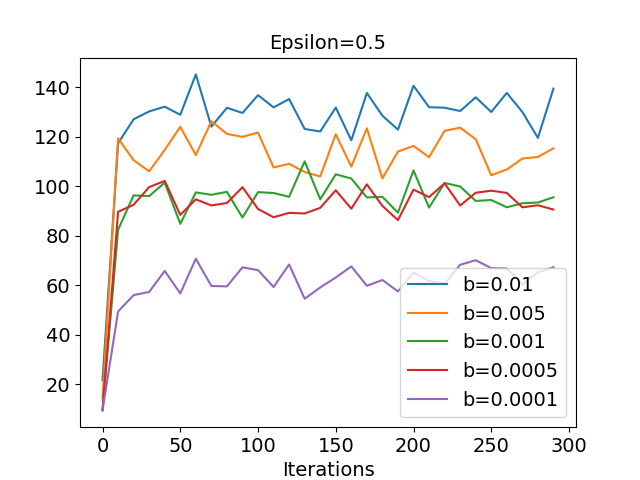}
    \end{minipage}
    \begin{minipage}{0.32\textwidth}
    \centering
    \includegraphics[width=\textwidth]{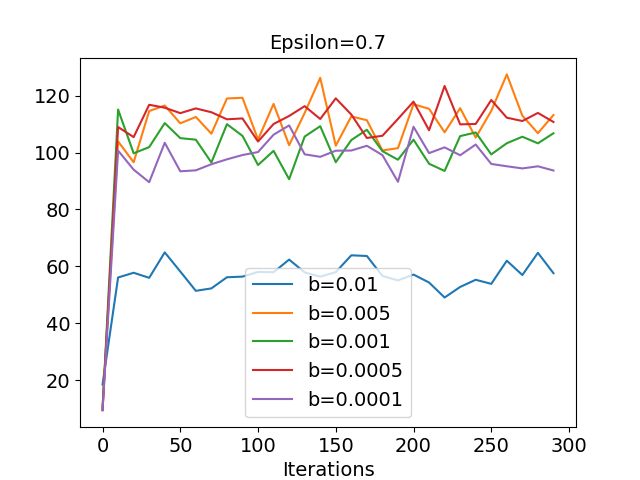}
    \end{minipage}
   
    \caption{Performance of MBS-QI with different values of threshold $b$}
    \label{fig:cp_b}
    \vspace{-0.2in}
\end{figure*}

In figure \ref{fig:cp_b} we show the trend that smaller $b$ works better for larger $\epsilon$ and larger $b$ works better for smaller $\epsilon$ in general. This can be explained in the following way: with a larger $\epsilon$ the data distribution is more exploratory and hence the probabilities on individual state, action pairs are smaller. So a the same threshold that performs well with low exploration now censors a much larger part of the state, action space, necessitating a smaller threshold as $\epsilon$ is increased. In general, we find that having the largest threshold which still retains a significant fraction of the state, action space is a good heuristic for setting the $b$ parameter.

\section{Details of Hopper-v3 Experiment}
\label{appendix:hopper}
In this section we introduce some missing details about the MBS-QL algorithm and the experimental details. We also attach the complete code for running our algorithm and the baseline algorithms in the supplemental materials.

\subsection{Implementation Details of MBS-QL}
MBS-QL algorithm is implemented based on the architecture of Batch-Constrained deep $Q$-learning (BCQ) \cite{fujimoto2019off} algorithm. More specifically, we use the similar Clipped Double Q-Learning (CDQ) update rule for the $Q$ learning part, and employ a similar variational auto-encoder to fit the conditional action distribution in the batch. We use an additional variational auto-encoder to fit the marginalized state distribution of the batch. To implement an actual $Q$ learning algorithm instead of an actor-critic algorithm, we did not sample from the actor in the Bellman backup but sample a larger batch from the fitted conditional action distribution. Algorithm \ref{alg:mbs_ql} shows the pseudo-code of MBS-QL to provide more details. We highlight the difference with the BCQ algorithm in red.

\begin{algorithm}[ht]
  \caption{MBS $Q$-learning (MBS-QL)}
  \label{alg:mbs_ql}
    \begin{algorithmic}
  \State {\bfseries Input:} Batch $D$, ELBO threshold $b$, maximum perturbation $\Phi$, target update rate $\tau$, mini-batch size $N$, max number of iteration $T$. Number of actions $k$.
  \State Initialize two Q network $Q_{\theta_1}$ and $Q_{\theta_2}$, policy (perturbation) model: $\xi_\phi$. ($\xi_\phi \in [-\Phi, \Phi]$), action VAE $G^a_{\omega_1}$ and state VAE $G^s_{\omega_2}$.
  \State Pretrain $G^s_{\omega_2}$:
  $\omega_2 \leftarrow \argmin_{\omega_2} ELBO(B; G^s_{\omega_2})$.
  \For{$t=1$ {\bfseries to} $T$}
  \State Sample a minibatch $B$ with $N$ samples from $D$. 
  \State $\omega_1 \leftarrow \argmin_{\omega_2} ELBO(B; G^a_{\omega_1})$.
  \State \textcolor{red}{Sample $k$ actions $a'_i$ from $G^a_{\omega_1}(s')$ for each $s'$.}
  \State Compute the target $y$ for each $(s,a,r,s')$ pair: 
  $$ y =  r + \gamma \textcolor{red}{\mathds{1}(ELBO(s';G^s_{\omega_2}) \ge b)} \left[  \max_{a'_i} \left( 0.75*\min_{j=1,2} Q_{\theta'_j} + 0.25*\max_{j=1,2} Q_{\theta'_j} \right) \right] $$
  \State $\theta \leftarrow \argmin_{\theta} \sum (y - Q_\theta(s,a))^2 $
  \State Sample $k$ actions $a_i$ from $G^a_{\omega_1}(s)$ for each $s$.
  \State $\phi \leftarrow \argmax_{\phi} \sum \max_{a_i} Q_{\theta_1} (s, a_i + \xi_\phi(s,a_i) ) $
  \State Update target network: $\theta' = (1-\tau)\theta' + \tau \theta$, $\phi' = (1-\tau)\phi' + \tau \phi$
  \EndFor
  \State {\bfseries When evaluate the resulting policy:} select action $a =\argmax_{a_i} Q_{\theta_1} (s, a_i + \xi_\phi(s,a_i) )$ where $a_i$ are $k$ actions sampled from $G^a_{\omega_1}(s)$ given $s$.
\end{algorithmic}
\end{algorithm}

In practice, the indicator function $\mathds{1}(ELBO(s';G^s_{\omega_2}) \ge b)$ is implemented by $\text{sigmoid}(100(ELBO(s';G^s_{\omega_2})- b))$ to provide a slightly more smooth target. The evidence lower bound (ELBO) objective function of VAE is:
\begin{align}
    ELBO(s;G^s_{\omega_2}) = \sum (s-\tilde s)^2 + D_\text{KL} (N(\mu,\sigma)||N(0,1))
\end{align}
where $\mu$ and $\sigma$ is sampled from the encoder of VAE with input $s$ and $\tilde s$ is sampled from the decoder with the hidden state generated from $N(\mu,\sigma)$. $ELBO(B;G^s_{\omega_2})$ is the averaged ELBO on the minibatch $B$. So does $G^a_{\omega_1}$. Note that this ELBO objective make the implicit assumption that the decoder's distribution is a Gaussian distribution with mean equals to the output of decoder network. So when we generate the sample $a'$ for computing $y$, we add a Gaussian noise $N(0,1.5)$ to the output of decoder network.

\subsection{Experimental details}
We collect the batch data using a near-optimal policy with two-level noise, which is the same setting as the ``Imperfect Imitation'' in the BCQ paper \cite{fujimoto2019off}. We first get the near-optimal policy by running DDPG in 15 random seeds and pick the best 5 resulting policies. Then with probability $0.3$ it samples uniformly from action space, with probability $0.7$ it adds a Gaussian noise $\mathcal{N}(0,0.3)$ to the action from DDPG policy.

For most of the hyper-parameters in Algorithm \ref{alg:mbs_ql}, we use the same value with the BCQ algorithm. We run all algorithms with $T=10^6$ gradient steps, and the minibatch size $N=100$ at each step. The number of action $k=100$. Target network update rate is $0.005$. The threshold $b$ of ELBO is selected as $2$-percentile of the $ELBO(s)$ in the whole dataset after pretrain the VAE.

A key hyper-parameter in BCQ is the maximum perturbation $\Phi$. It controls the amount of difference with the behavior cloning allowed by the BCQ algorithm. The results in Appendix D.1 of their paper \cite{fujimoto2019off} suggest a smaller $\Phi$ is always better and they use $\Phi=0.05$ to report their results. For our algorithm, we find that the performance is not monotonic with $\Phi$, which might suggest some more interesting behavior of our algorithm since it is less similar to behavior cloning. In Figure \ref{fig:hopper_phi} we report an ablation study of perturbation bound $\Phi$. We compare MBS-QL and BCQ with $\Phi=0.05, 0.1$ and $0.25$. For all values of $\Phi$ MBS-QL is better than BCQ with the same $\Phi$. For the results in main paper, we use $\Phi=0.1$ for MBS-QL and $\Phi=0.05$ for BCQ.

\begin{figure}[ht]
    \centering
    \includegraphics[width=0.8\textwidth]{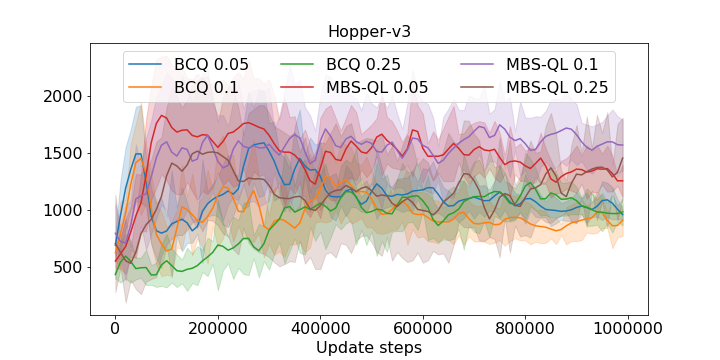}
    \caption{BCQ \cite{fujimoto2019off} and MBS-QL with different range of policy output ($\Phi$).}
    \label{fig:hopper_phi}
\end{figure}

\end{document}